\documentclass{article}

\usepackage{microtype}
\usepackage{graphicx}
\usepackage{subcaption}
\usepackage{booktabs} 

\usepackage{hyperref}

\usepackage[preprint]{icml2026}

\usepackage{amsmath}
\usepackage{amssymb}
\usepackage{mathtools}
\usepackage{amsthm}

\usepackage[capitalize,noabbrev]{cleveref}

\usepackage{amsmath,amsfonts,bm}

\def\eqref#1{equation~\ref{#1}}

\def\1{\bm{1}}

\def\vw{{\bm{w}}}
\def\vx{{\bm{x}}}
\def\vy{{\bm{y}}}

\def\mA{{\bm{A}}}

\def\mI{{\bm{I}}}

\def\mM{{\bm{M}}}

\def\mS{{\bm{S}}}

\def\mU{{\bm{U}}}
\def\mV{{\bm{V}}}
\def\mW{{\bm{W}}}
\def\mX{{\bm{X}}}

\def\mSigma{{\bm{\Sigma}}}

\DeclareMathAlphabet{\mathsfit}{\encodingdefault}{\sfdefault}{m}{sl}
\SetMathAlphabet{\mathsfit}{bold}{\encodingdefault}{\sfdefault}{bx}{n}

\def\gA{{\mathcal{A}}}
\def\gB{{\mathcal{B}}}

\def\gD{{\mathcal{D}}}

\def\gF{{\mathcal{F}}}

\def\gN{{\mathcal{N}}}

\def\gP{{\mathcal{P}}}
\def\gQ{{\mathcal{Q}}}

\def\sP{{\mathbb{P}}}
\def\sQ{{\mathbb{Q}}}
\def\sR{{\mathbb{R}}}
\def\sS{{\mathbb{S}}}

\newcommand{\E}{\mathbb{E}}

\newcommand{\R}{\mathbb{R}}

\DeclareMathOperator*{\argmin}{arg\,min}

\usepackage[flushmargin]{footmisc}

\usepackage{hyperref}
\usepackage{url}

\usepackage{mathtools}
\usepackage{amsmath}
\usepackage{bbm}
\usepackage{amsthm}
\usepackage{amssymb}
\usepackage{physics}
\usepackage{enumitem}

\theoremstyle{plain}
\newtheorem{theorem}{Theorem}[section]
\newtheorem{proposition}[theorem]{Proposition}
\newtheorem{lemma}[theorem]{Lemma}

\theoremstyle{definition}

\newtheorem{assumption}[theorem]{Assumption}

\newcommand{\cC}{\mathcal{C}}
\newcommand{\cN}{\mathcal{N}}

\DeclareMathOperator{\SimDiag}{SimDiag}

\def\diag{{\textbf{diag}}}
\def\Auto{{\mathrm{Auto}}}
\def\Extract{{\mathrm{Extract}}}
\DeclareMathOperator{\dist}{dist}

\newcommand{\dens}{\mathrm{dens}}
\DeclareMathOperator{\find}{find}
\DeclareMathOperator{\hop}{hop}

\usepackage[most]{tcolorbox}
\usepackage{needspace}

\newtcolorbox{examplebox}[1]{
  breakable,
  colback=white,
  colframe=black!60,
  boxrule=0.6pt,
  arc=2pt,
  left=6pt,right=6pt,top=6pt,bottom=6pt,
  fonttitle=\bfseries,
  title=#1
}

\newtcolorbox{modeloutbox}{
  breakable,
  colback=black!1,
  colframe=black!15,
  boxrule=0.4pt,
  arc=2pt,
  left=6pt,right=6pt,top=4pt,bottom=4pt,
  fontupper=\ttfamily\footnotesize
}

\newtcblisting{modelout}{
  listing only,
  breakable,
  colback=black!1,
  colframe=black!15,
  boxrule=0.4pt,
  arc=2pt,
  left=6pt,right=6pt,top=4pt,bottom=4pt,
  listing options={
    basicstyle=\ttfamily\footnotesize,
    columns=fullflexible,
    breaklines=true,
    breakatwhitespace=true
  }
}

\usepackage[textsize=tiny]{todonotes}

\icmltitlerunning{Beyond What Seems Necessary: Hidden Gains from Scaling Training-Time Reasoning Length under Outcome Supervision}

\begin{document}

\twocolumn[
  \icmltitle{Beyond What Seems Necessary: Hidden Gains from Scaling Training-Time Reasoning Length under Outcome Supervision}




  \begin{icmlauthorlist}
    \icmlauthor{Yihao Xue}{ucla}
    \icmlauthor{Allan Zhang}{ucla}
    \icmlauthor{Jianhao Huang}{ucla}
    \icmlauthor{Amit Sahai}{ucla}
    \icmlauthor{Baharan Mirzasoleiman}{ucla}
  \end{icmlauthorlist}

  \icmlaffiliation{ucla}{University of California, Los Angeles}

  \icmlcorrespondingauthor{Yihao Xue}{yihaoxue@g.ucla.edu}

  \icmlkeywords{Machine Learning, ICML}

  \vskip 0.3in
]



\printAffiliationsAndNotice{}  

\begin{abstract}
Training LLMs to think and reason for longer has become a key ingredient in building state-of-the-art models that can solve complex problems previously out of reach. Recent efforts pursue this in different ways, such as RL fine-tuning to elicit long CoT or scaling latent reasoning through architectural recurrence. This makes reasoning length an important scaling knob. In this work, we identify a novel phenomenon (both theoretically and experimentally): under outcome-only supervision, out-of-distribution (OOD) performance can continue improving as training-time reasoning length (e.g., the token budget in RL, or the loop count in looped Transformers) increases, even after in-distribution (ID) performance has saturated. This suggests that robustness may require a larger budget than ID validation alone would indicate. We provide theoretical explanations via two mechanisms: (i) self-iteration can induce a stronger inductive bias in the hypothesis class, reshaping ID-optimal solutions in ways that improve OOD generalization; and (ii) when shortcut solutions that work for ID samples but not for OOD samples persist in the hypothesis class, regularization can reduce the learned solution's reliance on these shortcuts as the number of self-iterations increases. We complement the theory with empirical evidence from two realizations of scaling training-time reasoning length: increasing the number of loops in looped Transformers on a synthetic task, and increasing token budgets during RL fine-tuning of LLMs on mathematical reasoning.\looseness=-1

\end{abstract}

\section{Introduction}




Large language models (LLMs) are increasingly trained to \emph{think longer}: they generate multiple intermediate steps before outputting the final answer. A prominent example is chain-of-thought (CoT) reasoning, where the model attempts to solve a problem step by step. In practice, such behavior is often induced through reinforcement learning (RL) fine-tuning \cite{schulman2015trust,schulman2017proximal,ouyang2022training,shao2024deepseekmath,yu2025dapo}, where the model is trained to generate substantial ``thinking'' tokens \cite{guo2025deepseek} during rollouts. In parallel, another line of research studies \emph{latent} reasoning carried out in continuous embedding space, for example via architectural recurrence such as looped Transformers \cite{giannou2023looped,saunshi2025reasoning} or continuous CoT \cite{hao2024training}. Together, these developments make \emph{reasoning length}, whether explicit (tokens) or latent (loops), a central scaling knob.\looseness=-1

Despite growing efforts to understand how CoT or looping affects a model's expressivity \cite{xu2024expressive,li2024chain,giannou2023looped}, and work on inference-time scaling of reasoning \cite{snell2024scaling,wu2024inference}, the effects of scaling this knob \emph{during training} remain underexplored. This question is particularly important under outcome-only supervision \cite{lightman2023let,jia2025we}---for example, in RL the reward is often assigned solely based on the final answer, and looped Transformers similarly provide no explicit guidance about how intermediate computation should proceed. In such settings, intermediate behavior is learned implicitly, making it crucial to understand what kinds of computation emerge through training. Essentially, scaling training-time reasoning length amounts to allowing more \emph{self-iteration} during training without specifying how those iterations should be carried out. We thus lack a learning-theoretic account of:\looseness=-1
\begin{list}{}{\leftmargin=0pt \rightmargin=0pt}
\vspace{-.2cm}
\item \emph{How does allowing the model to self-iterate more times during training (e.g., using a larger token budget during RL, or training looped Transformers with more loops) change the solution that learning selects? Can increasing reasoning length lead the model to solve the same task via a qualitatively different mechanism?}
\vspace{-.3cm}
\end{list}

Understanding this question matters in realistic settings where training and validation data are limited and distribution shift is expected at deployment. In such cases, in-distribution (ID) validation performance can be misleading: models can be ``right for the wrong reasons,'' and an apparent ID plateau can hide further gains in robust generalization. Indeed, we find that out-of-distribution (OOD) performance can continue improving with increased training-time reasoning length even after ID performance has saturated.

We provide theoretical explanations for these hidden OOD gains from a learning-theoretic perspective that views increasing training-time reasoning length as increasing the number of self-iterations, and analyzes how this changes the effective hypothesis class explored by the training algorithm. This shift can alter \emph{which} ID-optimal solution is selected, even when many solutions achieve the same ID error. We formalize two mechanisms: (i) self-iteration can induce a stronger inductive bias, reshaping the set of ID-optimal solutions in ways that improve OOD generalization; and (ii) when shortcut solutions persist in the hypothesis class, regularization can reduce the learned solution's reliance on these shortcuts as the number of self-iterations increases

We substantiate this phenomenon empirically through two realizations of training-time reasoning length. First, we scale \emph{latent} reasoning in looped Transformers by increasing the loop count. Training on synthetic algorithmic tasks, we observe that ID accuracy saturates at small loop counts while OOD accuracy continues to improve gradually. Second, we scale the length of \emph{explicit} CoT reasoning by increasing the token budget during RL finetuning on mathematical tasks, where ID and OOD correspond to different problem topics; we similarly find that OOD accuracy can keep improving well beyond the point where ID accuracy saturates.\looseness=-1

Overall, we provide theoretical justification for a practical recommendation: when training is outcome-supervised and distribution shift is expected after deployment, it can be beneficial to train with longer reasoning lengths (larger CoT token budgets in RL or more loops in looped Transformers) beyond what appears sufficient to saturate validation performance. In-distribution performance can be misleading, masking substantial OOD gains from further scaling.\looseness=-1

\section{Related Work}
\label{sec:related_work}

\textbf{RL post-training for reasoning.}
A key recent development is the use of variants of RL algorithms \cite{schulman2015trust,schulman2017proximal,ouyang2022training,shao2024deepseekmath,yu2025dapo} in post-training to produce models that generate longer CoT deliberation \citep{wei2022chain,kojima2022large} and thereby solve more complex problems. This paradigm is emphasized in prominent reasoning-model reports \cite{jaech2024openai,shao2024deepseekmath,guo2025deepseek,muennighoff2025s1}. In RL fine-tuning, training rollouts are typically capped by a maximum generation length, so the model learns to solve problems within a fixed token budget. The choice of this budget therefore becomes an important scaling knob: training with a larger budget typically yields more capable models \citep{khatri2025art,cai2025escaping,setlur2025e3}. \looseness=-1

\textbf{Latent reasoning via structural recurrence.}
In parallel with efforts to elicit longer explicit CoT, structural recurrence offers another route to scaling reasoning via latent ``thought". Looped Transformers apply the same Transformer backbone repeatedly with shared weights—an idea that can be traced back to \cite{dehghani2018universal}—and have recently attracted renewed interest due to their theoretical expressivity, including the ability to simulate arbitrary Turing machines \cite{de2024simulation,giannou2023looped}, express iterative algorithms \cite{yang2023looped,gao2024expressive,gatmiry2024can,gatmiry2024role,xu2024expressive}, and improved length generalization \cite{fan2024looped}. \citet{saunshi2025reasoning} shows that looped Transformers can exhibit strong reasoning performance and highlights empirical and theoretical connections to CoT, suggesting that looping can support a form of ``latent thought'' reasoning. This perspective is also closely related to the idea of continuous CoT \cite{hao2024training}.\looseness=-1

We connect these threads by viewing them as different realizations of \emph{self-iteration}. We offer the first learning-theoretic perspective on how the induced hypothesis class evolves as we scale the number of self-iterations, shaping what is learned on one distribution and how it generalizes OOD.\looseness=-1
\section{Training-Time Self-Iteration Can Yield OOD Gains Without ID Gains}

In this section, we develop theory to shed light on the phenomenon that \emph{OOD performance can continue improving as we increase training-time reasoning length, even when there is no visible ID improvement}. Increasing reasoning length---whether explicit (e.g., via CoT tokens) or latent (e.g., via the number of loops in looped Transformers)---can be viewed as different realizations of increasing self-iteration. Our key observation is that permitting more iterations before emitting the output changes how the model can produce its final answer. From a learning-theoretic perspective, this corresponds to a change in the \emph{hypothesis class} over which the training algorithm searches. A natural implication is that when the ID optimum can be achieved in multiple ways, changing the hypothesis class can steer learning toward a different ID-optimal solution that behaves differently under distribution shift. We will illustrate two mechanisms through which this effect can occur.\looseness=-1

\textbf{Notation.} We study a learning problem in which the goal is to learn a model \( f : \mathcal{X} \to \mathcal{Y} \) from a hypothesis class \( \mathcal{F}_k \), with ground-truth target function \( z : \mathcal{X} \to \mathcal{Y} \). Here \(k\) denotes the number of training-time iterations (e.g., the token budget for CoT or the number of loops in a looped transformer), and \( \mathcal{F}_k \) denotes the resulting hypothesis class. We consider a population-level setting in which the model is trained with respect to a distribution \( \mathcal{P} \) over \( \mathcal{X} \) and is subsequently evaluated on a different distribution \( \mathcal{Q} \).\looseness=-1

Understanding our phenomenon therefore requires analyzing how \( \mathcal{F}_k \) evolves as \( k \) increases, and how this evolution interacts with the learning algorithm and the data distribution. We will provide several characterizations in which increasing \(k\) yields OOD gains despite little or no ID improvement, through different mechanisms.

\subsection{Self-Iteration Can Strengthen Inductive Bias in the Hypothesis Class and Yield OOD Gains}
\label{subsec:iteration_reduce}

It is often tempting to think that self-iteration increases expressivity—e.g., it may enable a model to implement a broader class of serial computations \cite{merrill2023expressive,li2024chain}. However, increased expressivity alone cannot explain why we obtain solutions with better OOD performance but the same ID performance. We find that self-iteration can steer the solution toward one that is better aligned with the underlying target by inducing a more favorable inductive bias in the hypothesis class, rather than merely expanding its capacity. Indeed, if the base function class $\mathcal{F}_1$ is already sufficiently expressive—as is commonly the case for large neural networks—then self-iteration need not increase capacity; e.g., the $k$-fold self-iterate imposes the constraint that the function must admit a $k$-th compositional root. In the remainder of this subsection, we present two such examples: one where self-iteration shrinks the hypothesis space, and another where self-iteration shifts it—both in ways that are more favorable for OOD generalization, while remaining equally optimal in-distribution.\looseness=-1

\subsubsection{Example: Self-Iterating the Full Hypothesis Class on a Cycle Task}
\label{subsubsec:composition_stronger_bias}

We illustrate the effect of self-iteration through a deliberately extreme toy example. The key ``magic'' is that the base hypothesis class is maximally expressive: \(\gF_1\) contains \emph{every} possible function and is therefore completely unconstrained, much like an overparameterized neural network that can fit the data in many incompatible ways. Nevertheless, when we restrict the learner to \(k\)-fold self-compositions of functions from this same class, the resulting hypothesis class \(\gF_k\) acquires a strong implicit structure---turning an unconstrained class into one that selects a uniquely correct solution and generalizes perfectly to the OOD data.

\textbf{A cycle task.}
Fix an integer \(n\ge 3\) and define the domain $
\sS \coloneqq \{0,1,\dots,n-1\}$.
Fix any held-out point \(q\in\sS\), and let $
\sQ \coloneqq \{q\}$ and $
\sP \coloneqq \sS\setminus \sQ$. Let \(\gP\) be the uniform distribution over \(\sP\) and \(\gQ\) be the uniform distribution over \(\sQ\)
(a point mass at \(q\)). Let the ground-truth target \(z\!:\!\sS\!\to\!\sS\) be some cyclic permutation on \(\sS\).\looseness=-1

Given a function \(f:\sS\!\to\!\sS\) and a distribution \(\gD\) over \(\sS\), define the \(0\text{-}1\) loss $
L_{\gD}(f) \coloneqq \E_{x\sim \gD}\,\mathbf{1}\{f(x)\neq z(x)\}$. Given a hypothesis class
$\mathcal{H}$, we minimize the ID population loss \(L_{\gP}\), namely $
\argmin_{f\in\mathcal{H}} L_{\gP}(f)$. We then evaluate the resulting minimizer on \(\gQ\) as a measure of OOD performance.\looseness=-1

We take the base hypothesis class \(\gF\) to be the \emph{entire} class of endomorphisms on \(\sS\), and define \(\gF_k\) to be the class of \(k\)-fold self-compositions of functions in \(\gF\):
\[
\gF \coloneqq \{\,f:\sS\to\sS\,\},
\qquad
\gF_k \coloneqq \gF^{\circ k} \;=\; \{\,f^{\circ k}: f\in\gF\,\}.
\]
Here \(f^{\circ k}\) denotes the \(k\)-fold self-composition \(f\circ f \circ \cdots \circ f\).

\textbf{OOD failure at \(k\!=\!1\).}
With \(k\!=\!1\), the ID minimizer can fit \(\gP\) perfectly while behaving arbitrarily on $\gQ$.\looseness=-1

\begin{proposition}\label{prop:cycle_k1_failure}
In the task above, there exists \(\hat f_1\! \in\! \argmin_{f\in\gF_1} L_{\gP}(f)\) s.t.  $
L_{\gP}(\hat f_1)=0 $ and $
L_{\gQ}(\hat f_1)=1$.\looseness=-1
\end{proposition}
\vspace{-.2cm}
\begin{proof} Immediate since \(\gF_1=\gF\) contains all functions \(\sS\to\sS\): one can interpolate \(\sP\) and set the value at \(q\) arbitrarily.
\end{proof}

\textbf{Most sufficiently large \(k\) enforce a unique OOD-optimal ID minimizer.}
We next show that, despite \(\gF\) being completely unconstrained, for a large set of \(k\) the iterated class \(\gF_k\)
induces a much stronger inductive bias: the ID minimizer becomes \emph{unique} and equals the ground-truth \(z\), and hence is OOD-optimal.
We quantify ``most'' via (lower) natural density: for a set \(A\subseteq\mathbb{N}\), define
\[
\dens(A)\coloneqq \liminf_{N\to\infty}\frac{|A\cap\{1,2,\dots,N\}|}{N}.
\]
Intuitively, it is the asymptotic fraction of integers in \(A\).\looseness=-1

\begin{theorem}\label{thm:high_density_unique_ood_clean} For any \(\varepsilon>0\), there exists a prime $n$ and a set of ``good" values of $k$
\[
\mathcal{K}\coloneqq \{\,k\geq n-1:\ \gcd(k,n)=1 \ \text{and}\ \gcd(k,n-1)>1\,\}
\]
such that the following holds. (i) \(\mathcal{K}\) has high density: $
\dens(\mathcal{K}) \!\ge\! 1\!-\!\varepsilon$. (ii) For every choice of held-out point \(q\in\sS\), every cyclic-permutation target \(z\), and every \(k\in\mathcal{K}\), the ID minimizer over \(\gF_k\) is unique and equals the ground truth: $
\argmin_{f\in\gF_k} L_{\gP}(f) \!=\! \{z\}$. In particular, the unique minimizer \(\hat f=z\) achieves zero OOD error, i.e., \(L_{\gQ}(\hat f)=0\).
\end{theorem}
\noindent\textbf{Proof road map.}
The proof of Theorem~\ref{thm:high_density_unique_ood_clean} is deferred to Appendix~\ref{apdx: extension_cycle}.
We first establish a structural uniqueness criterion for good interpolants at fixed \(n\) and \(k\)
(Thm.~\ref{thm:unique-good}), using functional-graph properties of large iterates and a characterization of
interpolants realizable as \(k\)-fold iterates.
Intuitively, \(\gcd(k,n)\!=\!1\) ensures that the target \(z\) lies in \(\gF_k\), while \(\gcd(k,n\!-\!1)\!>\!1\) eliminates the only spurious interpolant, forcing uniqueness.
We then use density calculations and Dirichlet’s theorem on primes in arithmetic progressions to pick a prime \(n\)
for which the gcd conditions hold for a set of exponents \(k\!\ge \!n\!-\!1\) with density at least \(1\!-\!\varepsilon\).\looseness=-1

As an example, for $\epsilon\!=\!0.25$, we can take \(n\!=\!2311\). Then Thm~\ref{thm:high_density_unique_ood_clean} implies that for at least a \(0.75\) fraction of integers \(k\!\geq 2310\),
learning over \(\gF_k\) achieves zero error both ID and OOD.
In contrast, when \(k\!=\!1\), Prop.~\ref{prop:cycle_k1_failure} shows that there exists an ID minimizer with maximal OOD error \(L_{\gQ}\!=\!1\).\looseness=-1

\subsubsection{Example: In-context coefficient prediction with a linear transformer}
\label{sec:icgd_infinite_context_linreg}

We present an example of a linear transformer trained to perform in-context coefficient prediction \cite{huang2025transformers} for linear regression, trained with multiple self-iterations before producing its final prediction. We show how increasing the number of self-iterations shifts the induced hypothesis class toward a generic (in-context) gradient descent solver, yielding OOD improvements even when the ID loss remains zero. We note that our notation is tailored to mimic autoregressive generation, but in this continuous token-space setting it can be equivalently rewritten in a form more reminiscent of looped transformers \cite{gatmiry2024can}.
\looseness=-1

\textbf{Input sequence structure.}
Each input sequence encodes a linear regression problem identified by a tuple $(\mSigma,\bm{\beta})$, where $\vx \!\sim\! \gN(0,\mSigma)$ and $y\!=\!\vx^\top \bm{\beta}$.
Given $(\mSigma,\bm{\beta})$, we sample $n$ i.i.d.\ pairs $\{(\vx_i,y_i)\}_{i=1}^n$ and assemble them into an input
matrix $\mS^{(n)}_{\mSigma,\bm{\beta}}\in \sR^{d_e\times (n+1)}$ with $d_e=2d+2$ as follows.
Let $\mX=[\vx_1,\dots,\vx_n]\in\sR^{d\times n}$ and $\vy=(y_1,\dots,y_n)^\top\in\sR^n$. Define \looseness=-1
\begin{equation}\label{eq:linreg_sequence}
\vspace{-.1cm}
\setlength{\arraycolsep}{4pt}
\renewcommand{\arraystretch}{0.9}
\mS^{(n)}_{\mSigma,\bm{\beta}}
\coloneqq S\!\left(\{(\vx_i,y_i)\}_{i=1}^n\right)
\coloneqq
\begin{bmatrix}
\tfrac{1}{\sqrt{n}}\mX & \mathbf{0}_{d\times 1}\\
\tfrac{1}{\sqrt{n}}\vy^\top & 0 \\
\mathbf{0}_{d\times n} & \mathbf{0}_{d\times 1}\\
\mathbf{0}_{1\times n} & 1
\end{bmatrix}.
\end{equation}
The last column is a ``query'' token with a constant $1$ in the final coordinate. Note that
$\mS^{(n)}_{\mSigma,\bm{\beta}}$ is a random variable conditioned on $(\mSigma,\bm{\beta})$, while $S(\cdot)$ is a deterministic function of the sampled dataset.
The goal is to take \( \mS^{(n)}_{\mSigma,\bm{\beta}} \) as the prompt and infer \( \bm{\beta} \) from the in-context examples.

\textbf{Linear transformer with reduced weights.}
We consider a one-layer linear self-attention module with a residual connection and reduced weights
(similar to that in \cite{huang2025transformers}). For any input matrix $\mM\in\sR^{d_e\times \ell}$, define\looseness=-1
\begin{align}
\vspace{-.2cm}
\label{eq:linear_transformer_T}
T(\mM\mid \mV,\mW)
= \mM + \tilde{\mV}\,\mM \mM^\top \tilde{\mW}\,\mM,
\end{align}
where the only trainable blocks are $\mV,\mW\in\sR^{d\times d}$ and
\begin{align}
\nonumber
\tilde{\mV} &
=
\begin{bmatrix}
\mathbf{0}_{(d+1)\times d} & \mathbf{0}_{(d+1)\times (d+2)} \\
\mV & \mathbf{0}_{d\times (d+2)}\\
\mathbf{0}_{1\times d} & \mathbf{0}_{1\times (d+2)}
\end{bmatrix},
\end{align}
\begin{align}
\nonumber
\tilde{\mW}
& =
\begin{bmatrix}
\mathbf{0}_{d\times (d+1)} & \mW & \mathbf{0}_{d\times 1} \\
\mathbf{0}_{1\times (d+1)} & \mathbf{0}_{1\times d} & -1 \\
\mathbf{0}_{(d+1)\times (d+1)} & \mathbf{0}_{(d+1)\times d} & \mathbf{0}_{(d+1)\times 1}
\end{bmatrix}.
\end{align}
We write the last token (last column) of the output as
\begin{align}
T(\mM\mid \mV,\mW)_{[:,\ell]}
= \mM_{[:,\ell]} + \tilde{\mV}\,\mM \mM^\top \tilde{\mW}\,\mM_{[:,\ell]}.
\end{align}

\textbf{Autoregressive generation and extraction.}
We define a one-step autoregressive extension by appending the updated last token. Given $(\mV,\mW)$, for any $\mM\in\sR^{d_e\times \ell}$:\looseness=-1
\begin{align}
\nonumber
\Auto(\mM\mid \mV,\mW)
=
\begin{bmatrix}
\mM \quad T(\mM\mid \mV,\mW)_{[:,\ell]}
\end{bmatrix}. 
\end{align}
We also define an extraction map $\Extract(\cdot)$, which, conceptually, abstracts how a model’s final answer is obtained
from its generated sequence, enabling direct comparison with the reference answer. In our case, we specialize it to
\begin{align}
\label{eq:extract_def}
\Extract(\mM) \;=\; \mM_{[d+2:2d+1,\,\ell]}\in\sR^d.
\end{align}
i.e., we extract rows $d+2,\dots,2d+1$ of the last token. Given a generation budget $k$, the induced predictor is
\begin{align}
\label{eq:fk_def}
f_k(\cdot\mid \mV,\mW) \;=\; \Extract\circ \Auto^{\circ k}(\cdot\mid \mV,\mW),
\end{align}
whose output lies in the same space as $\bm{\beta}$, the target coefficients we aim to predict.

\textbf{In-context GD-implementing parameters.}
Given a dataset $\cC\!=\!\{(\vx_i,y_i)\}_{i=1}^n$, consider the squared loss $
R(\vw;\cC)\!\coloneqq \!\frac{1}{n}\sum_{i=1}^n(\vx_i^\top \vw-y_i)^2$. For a step size $\eta\!>\!0$, one-step gradient descent (GD) on $R$ is $
\vw_{t+1}=\vw_t-\eta\nabla R(\vw_t;\cC)$. On the other hand, given $\cC$, we form the input sequence $S(\cC)$ as in \eqref{eq:linreg_sequence} and run the autoregressive model.
For parameters $(\mV,\mW)$, define the extracted prediction at generation step $t$ by\looseness=-1
\[
\hat{\vw}_t(\cC;\mV,\mW)\coloneqq \Extract\!\left(\Auto^{\circ t}\big(S(\cC)\mid \mV,\mW\big)\right).
\]
We say that $(\mV,\mW)$ \emph{implement GD updates} if there exists a step size $\eta>0$ such that, for every dataset $\cC$ and every $t\!\ge\! 0$,\looseness=-1
\[
\hat{\vw}_{t+1}(\cC;\mV,\mW)
=
\hat{\vw}_t(\cC;\mV,\mW)-\eta\nabla R\!\left(\hat{\vw}_t(\cC;\mV,\mW);\cC\right).
\]
Let $\Theta_{\rm GD}$ denote the set of such \emph{generic} GD-implementing $(\mV,\!\mW)$ pairs (i.e., the above equality holds uniformly over all datasets).
Appendix~\ref{subsec:proof_two_point_iso_main} shows that this is equivalent to\looseness=-1
\[
\label{eq:theta_gd_def}
\Theta_{\rm GD}
=\Big\{\,(-\eta \mI_d,\ \mI_d)\ \big|\ \eta>0\,\Big\}.
\]

\textbf{Infinite-context limit and loss.}
For each $n$, consider the model's prediction $f_k(\mS^{(n)}_{\mSigma,\bm{\beta}}\mid \mV,\mW)$ on the input sequence $\mS^{(n)}_{\mSigma,\bm{\beta}}$.
We study the infinite-context limit $n\to\infty$ for any finite $k$.
Under our data model, the empirical moments in
$\mS^{(n)}_{\mSigma,\bm{\beta}}\mS^{(n)\top}_{\mSigma,\bm{\beta}}$ converge to their population counterparts, implying that
$f_k(\mS^{(n)}_{\mSigma,\bm{\beta}}\mid \mV,\mW)$ converges almost surely to a deterministic limit conditioned on $(\mSigma,\bm{\beta})$ (see Appendix~\ref{subsec:infinite_context_explicit}). For a distribution $\gD$ over $(\mSigma,\bm{\beta})$, we then define the infinite-context population loss as\looseness=-1
\begin{align}
\nonumber
L_{\gD}(f_k)
\;:=\;
\E_{(\mSigma,\bm{\beta})\sim\gD}\;
\big\| \lim_{n\to\infty} f_k(\mS^{(n)}_{\mSigma,\bm{\beta}}\mid \mV,\mW)-\bm{\beta}\big\|^2.
\end{align}

\textbf{Function class.}
For each generation budget $k$, define
\[
\gF_k
:=\big\{ f_k(\cdot\mid \mV,\mW)\,:\, \SimDiag_{\sR}(\mV,\mW)\big\},
\]
where $\SimDiag_{\sR}(\mV,\mW)$ means that $\mV,\mW\in\sR^{d\times d}$ are simultaneously orthogonally diagonalizable over $\sR$.
We impose this restriction to simplify the analysis.
Note that this restriction still contains the entire GD-implementing family.\looseness=-1

\textbf{Data distributions.}
$\gP$ and $\gQ$ are different distributions over $(\mSigma,\bm{\beta})$.
We assume $\mSigma$ and $\bm{\beta}$ are independent under both distributions, and $\bm{\beta}\sim \gN(0,\mI_d)$.
Under $\gP$, $\mSigma$ is uniformly distributed over $\{\gamma_1\mI_d,\gamma_2\mI_d\}$ with
$\gamma_1,\gamma_2>0$ and $\gamma_1\neq\gamma_2$.
Under $\gQ$, $\mSigma=\gamma_{\gQ}\mI_d$ for some $\gamma_{\gQ}\in(0,\gamma_1+\gamma_2)\setminus\{\gamma_1,\gamma_2\}$.

\textbf{A mechanistic explanation: self-iteration steers learning toward an in-context GD solver.} We focus on even $k$. For the ID task, we can always achieve a perfect fit for any even \(k\ge 2\). However, the resulting OOD loss decreases strictly as $k$ increases, so larger $k$ provides a real benefit that is not visible from ID performance alone. Our analysis also provides a mechanistic explanation: as $k$ increases, ID minimizers behave increasingly like a generic GD solver, illustrating how self-iteration shapes the hypothesis class to favor more generalizable behavior. See proof in Appendix \ref{subsec:proof_two_point_iso_main}.\looseness=-1
\begin{theorem}
\label{thm:two_point_iso}
We consider even $k\ge 2$. In the above setting, let
$f_k^\star\in\argmin_{f\in\gF_k} L_{\gP}(f)$ be any minimizer on $\gP$.
Let $
a:=\frac{\gamma_1+\gamma_2}{2}$, $ b:=\frac{|\gamma_1-\gamma_2|}{2}$.
Then:
\begin{enumerate}
\item \textbf{(ID fit is always perfect)} $L_{\gP}(f_k^\star)=0$.

\item \textbf{(Increasing $k$ improves OOD loss)} We have $
L_{\gQ}(f_k^\star)
= d\left(\frac{(a-\gamma_{\gQ})^k-b^k}{a^k-b^k}\right)^2$, which decreases strictly as $k$ increases over even integers.\looseness=-1

\item \textbf{(ID minimizers become increasingly similar to an in-context GD solver as $k$ grows)} Define the parameter distance $ \dist\big((\mV,\mW),\Theta_{\rm GD}\big) \!\coloneqq\!\inf_{(\mV',\mW')\in\Theta_{\rm GD}}\Big(\|\mV\!-\!\mV'\|_F^2+\|\mW\!-\!\mW'\|_F^2\Big)$. Then, for every ID minimizer $f_k^\star$ with parameters $(\mV^\star,\mW^\star)$, $
\dist\big((\mV^\star,\mW^\star),\Theta_{\rm GD}\big)
= d\left(\frac{b}{a}\right)^{2k}$,
which decreases as $k$ increases over even integers.\looseness=-1
\end{enumerate}
\end{theorem}

\subsection{Regularization Can Yield OOD Gains When Shortcuts Persist in the Hypothesis Class}
\label{subsec:decomp_shortcut}

When discussing OOD generalization, a commonly invoked concept is that of a
\emph{shortcut}. In general, a shortcut refers to a pattern that a model can exploit to fit the target on the
training distribution \( \mathcal{P} \), but which is not reliable or exploitable on an unseen distribution
\( \mathcal{Q} \) \cite{geirhos2020shortcut}.
From this perspective, in the setting analyzed in Section~\ref{subsec:iteration_reduce}, one can view
shortcut solutions as eventually disappearing as we increase \(k\). However, shortcuts may \emph{persist}
in the hypothesis class even as we iterate. We therefore characterize another mechanism that can still
yield OOD gains even when ID performance does not improve. The key requirement is that the learning
algorithm employs an appropriate regularization—one that penalizes shortcut solutions more strongly than
the structured, generalizable component. We formalize this mechanism below.\looseness=-1

Let \( \mathcal{P}, \mathcal{Q} \) be distributions over \( \mathcal{X} \), and let
\( z : \mathcal{X} \to \mathcal{Y} \) be a fixed target function, where \( \mathcal{Y} \) is a (real) Hilbert space.
For any measurable \( f : \mathcal{X} \to \mathcal{Y} \) and any distribution \( \mathcal{D} \) over \( \mathcal{X} \), define $
L_{\mathcal{D}}(f) \coloneqq \mathbb{E}_{x\sim\mathcal{D}}\|f(x)-z(x)\|^2$. 

We assume the hypothesis class admits a decomposition into a structured part \( \mathcal{G}_k \) and a shortcut part \( \mathcal{H}_k \):
\begin{equation}
\label{eq:decomp}
\mathcal{F}_k \;=\; \{\, g + h \;:\; g \in \mathcal{G}_k,\; h \in \mathcal{H}_k \,\},
\end{equation}
with \( 0 \!\in\! \mathcal{G}_k \), \( 0 \!\in\! \mathcal{H}_k \).
We learn on \( \mathcal{P} \) with regularization \(R\):\looseness=-1
\begin{equation}
\label{eq:erm}
\begin{split}
(\hat g,\hat h)\in
\argmin_{g\in\mathcal{G}_k,\; h\in\mathcal{H}_k}
\Big( L_{\mathcal{P}}(g+h) + R(g,h)\Big),
\end{split}
\end{equation}
with \(R(g,h)\!\coloneqq\! R_G(g)\!+\! R_H(h)\) \footnote{
For a clean presentation, we assume the regularizer decomposes additively. More generally, one can allow a bounded cross-term and obtain analogous results with extra terms scaling with its bound.\looseness=-1.} Here,
\( R_G,R_H \) are nonnegative functionals with \( R_G(0)\!=\!R_H(0)\!=\!0 \).
We evaluate the learned predictor \( \hat f \coloneqq \hat g+\hat h \).

\begin{assumption}\label{as:decomp}
We formalize \( \mathcal{G}_k \) as a structured, well-constrained family that behaves similarly on ID and OOD data, and \( \mathcal{H}_k \) as a weakly constrained family that can exploit shortcuts. We also make assumptions on the regularization.\looseness=-1

\textbf{Well-structured part: distribution shift control on \( \mathcal{G}_k \).}
There exists \( \gamma\ge 1 \) such that for all \( g\in\mathcal{G}_k \),
\begin{equation}
\label{eq:shift_control_G}
\frac{1}{\gamma} L_{\mathcal{P}}(g)\ \le\ L_{\mathcal{Q}}(g)\ \le\ \gamma L_{\mathcal{P}}(g),
\end{equation}
capturing the idea that \( \mathcal{G}_k \) contains transferable mechanisms.\looseness=-1

\textbf{Shortcut part: failure of OOD residual fit in \( \mathcal{H}_k \).}
For any minimizer \((\hat g,\hat h)\) of \eqref{eq:erm}, define the residual
\(
\hat r \coloneqq z-\hat g.
\)
We assume there exists \(C>0\) such that
\begin{equation}
\label{eq:shortcut_failure}
\mathbb E_{x\sim\mathcal Q}\|\hat h(x)-\hat r(x)\|^2
\;\ge\;
C\,\mathbb E_{x\sim\mathcal Q}\|\hat r(x)\|^2 .
\end{equation}
Intuitively, \( \mathcal H_k \) captures non-transferable ``shortcut'' mechanisms: although \( \hat h \) may help
reduce the objective on \( \mathcal P \), it cannot predict the remaining target residual
\( \hat r=z-\hat g \) on \( \mathcal Q \) beyond a constant-factor improvement over doing nothing.\looseness=-1

\textbf{Regularization with the right bias.}
In practice, learning typically induces strong regularization. Beyond explicit regularization such as weight decay
and the implicit bias of gradient-based optimization, popular training algorithms introduce substantial additional
regularization---e.g., RL commonly uses KL constraints, objective clipping, and trust-region updates. We hypothesize
that these effects bias learning toward more structured solutions, in the sense that they effectively control the
shortcut component without overly penalizing the structured component.\looseness=-1

\textbf{(i) Norm control for the shortcut regularizer.}
There exist constants \( a_{\mathcal P},a_{\mathcal Q}\ge 0 \) such that for all \( h\in\mathcal{H}_k \),
\begin{equation}
\label{eq:norm_control}
\mathbb{E}_{x\sim \mathcal D}\|h(x)\|^2 \le a_\mathcal D\,R_H(h),
\qquad \mathcal D\in\{\mathcal P,\mathcal Q\}.
\end{equation}
A simple example is a norm-based regularization, e.g., \( R_H(h)\propto \|h\|_{\mathcal H}^2 \) for an appropriate
function norm.\looseness=-1

\textbf{(ii) Bounded regularization for a structured fit.} The regularization does not overly penalize the structured component \( \mathcal G_k \): there exists \( \rho_k\!\ge\! 0 \) and
\(g_k^\star\in\argmin_{g\in\mathcal G_k} L_{\mathcal P}(g)\) s.t.\looseness=-1
\begin{equation}
\label{eq:small_reg_g}
R_G(g_k^\star)\le \rho_k,
\end{equation}
\end{assumption}

\textbf{Decoupled scaling of ID and OOD performance.}
In this decomposition view, ID performance is primarily governed by how well the shortcut component can fit the
training distribution, whereas OOD performance is governed—both from below and above—by how well the structured
component can fit, as shown below.\looseness=-1
\begin{theorem}\label{thm:main} Under Assumption~\ref{as:decomp}, define $ \epsilon_k \!\coloneqq\! \min_{g\in\mathcal{G}_k} L_{\mathcal P}(g)$, which measures how well the structured component \(\mathcal G_k\) can fit the target in-distribution, and define $\delta_{k} \!\coloneqq\! \min_{h\in\mathcal H_k}\Bigl(L_{\mathcal P}(h)\!+\!\,R_H(h)\Bigr)$, which measures the best regularized ID objective achievable using only the shortcut component. For any ID minimizer \(\hat f\) (\eqref{eq:erm}): \looseness=-1 \begin{enumerate} \vspace{-.2cm} \item \textbf{ID upper bound.} \(L_{\mathcal P}(\hat f)\le \delta_{k}\). \vspace{-.2cm} \item \textbf{OOD sandwich bound.} $ C_1\,\epsilon_k \;\le\; L_{\mathcal Q}(\hat f) \;\le\; C_2\,\bigl(\epsilon_k+ \rho_k\bigr)$, where \(C_1\coloneqq \frac{C}{\gamma}\) and $ C_2 \coloneqq 4\gamma + 4\gamma a_{\mathcal P}+2a_{\mathcal Q}$. \end{enumerate} \vspace{-.2cm} \end{theorem}

Consider a regime in which the shortcut channel \( \mathcal{H}_k \) is highly expressive and achieves a small regularized ID objective even for small \(k\), so that \( \delta_{k} \) decays rapidly. In contrast, the structured channel \( \mathcal{G}_k \) may require larger \(k\) for its expressive power to manifest, so \( \epsilon_k \) decays much more slowly. In this regime, the theorem implies that the ID error---upper bounded by \( \delta_{k} \)---can decrease quickly and plateau early, while the OOD error---lower bounded by \(\epsilon_k\) up to constants---can improve no faster than the rate at which \(\epsilon_k\) decreases.\looseness=-1

\textbf{Role of regularization in the competition between $\mathcal G_k$ and $\mathcal H_k$.}
The tightness of the OOD upper bound in Thm.~\ref{thm:main} is affected by \(C_2\) and  \(\rho_k\).
A tighter bound favors (i) sufficiently strong regularization on the shortcut channel, reflected by reasonable
norm-control constants \(a_{\mathcal P},a_{\mathcal Q}\), and (ii) mild
regularization on the structured channel, reflected by a small \(\rho_k\) (i.e., a small penalty on the structured fit). Mechanistically, as \(k\) grows and \(\mathcal G_k\) becomes expressive
enough to compete with shortcut fits, this regularization profile suppresses the shortcut component \(\hat h\)
while leaving a low-complexity structured fit available, biasing minimizers toward \(\mathcal G_k\). As a result,
reliance on \(\mathcal H_k\) decreases and OOD performance can continue to improve even after the ID error has
plateaued.\looseness=-1

\subsubsection{A concrete example with slow structured gains and fast shortcuts}\label{subsec:construction_shortcut}

For the purpose of illustrating the general result in Theorem~\ref{thm:main}, we present an admittedly stylized instantiation designed to make the structured/shortcut decomposition and the role of regularization transparent. Although analyzing highly complex models remains challenging at present, the same qualitative behavior is empirically supported in Sec.~\ref{sec:exp}.\looseness=-1

For a sequence \(s=(e_1,\dots,e_T)\), let \(s[i]\) denote its \(i\)-th element and \(s[-1]\) its last element; when \(T=1\), we identify \(s\) with \(e_1\) when the context is clear. Each element in the sequence is a token represented as a vector in \(\mathbb R^d\), and the prompt is a single token \(x\in\mathbb R^d\). For a vector \(\mathbf v\in\mathbb R^d\), let \(\mathbf v_{i:j}\) denote coordinates \(i\) through \(j\).\looseness=-1

\textbf{An autoregressive model.}
We define \(\mathrm{Auto}(s)\!\coloneqq \!\mathrm{Append}(s,\,N(s))\), where \(N\) is a next-token generator that maps the current sequence to the next token.
After \(k\) steps, the model outputs \(f_k(x)\coloneqq \mathrm{Extract}(\mathrm{Auto}^{\circ k}(x))\), where \(\mathrm{Extract}(\cdot)\) maps a token sequence to \(\mathcal Y\).

\textbf{An iterative channel and a non-iterative channel.}
Next, we specify the next-token generator \( N \).
Let \( a \in \mathcal{A} \) be a function \( \mathbb{R}^p \to \mathbb{R}^p \) and
\( b \in \mathcal{B} \) be a function \( \mathbb{R}^p \to \mathbb{R}^l \).
Assume \looseness=-1
\begin{equation}
\label{eq:next_token_ab}
    N(s)_{1:p} = a\bigl(s[-1]_{1:p}\bigr), \; N(s)_{p+1:p+l} = b\bigl(s[1]_{1:p}\bigr).
\end{equation}
with all other coordinates arbitrary as they play no role in the construction. Thus \(a\) is iterated across steps (so its expressivity can grow with \(k\)), while \(b\) depends only on the initial token and does not benefit from additional iterations. They write to disjoint coordinate blocks, so they do not interfere.\looseness=-1

\textbf{Resulting function class decomposition.}
Let \( \psi_1:\mathbb{R}^p \to \mathcal{Y} \) and \( \psi_2:\mathbb{R}^l \to \mathcal{Y} \). We define the extraction rule\looseness=-1
\[
\mathrm{Extract}(s)
\coloneqq
\psi_1\!\bigl(s[-1]_{1:p}\bigr) + \psi_2\!\bigl(s[-1]_{p+1:p+l}\bigr).
\]
The prediction induced by the model is then given by\looseness=-1
\begin{equation}
\label{eq:fk_decompose_ab}
f_k(x) = \psi_1\!\Bigl(a^{\circ k}\bigl(x_{1:p}\bigr)\Bigr) + \psi_2\!\Bigl(b\bigl(x_{1:p}\bigr)\Bigr).
\end{equation}
The induced hypothesis class decomposes as
\[
\mathcal{G}_k = \{\, \psi_1 \circ a^{\circ k} \circ \pi : a \in \mathcal{A} \,\},
\;
\mathcal{H}_k = \{\, \psi_2 \circ b \circ \pi : b \in \mathcal{B} \,\}.
\vspace{-.1cm}
\]

\textbf{Data.} Let \( \mathcal{Y}=\mathbb{R} \) and \( l=1 \).
Let the target function \( z \) be linear in the first \( p \) coordinates:\looseness=-1
\[
z(x)=\pi(x)^\top\theta,
\qquad
\theta=(\theta_1,\dots,\theta_p)^\top \in \mathbb{R}^p_{\ge 0}.
\]
Let \(\mathcal U\coloneqq\{\pm 1\}^p\) be the hypercube, with two disjoint subsets
\[
\mathcal U_P\coloneqq\{u\in\mathcal U:\ u_1=+1\},\;
\mathcal U_Q\coloneqq\{u\in\mathcal U:\ u_1=-1\}.
\]
We define distributions on inputs \(x\in\mathbb R^d\) by first sampling \(u\!\in\!\mathbb R^p\), then setting \(x_{1:p}=u\); the remaining coordinates can be arbitrary and play no role in the construction.
Let \(\mathcal P\) be the distribution obtained by sampling \(u\!\sim\!\mathrm{Unif}(\mathcal U_P)\), and let
\(\mathcal Q\) be obtained by sampling \(u\!\sim\!\mathrm{Unif}(\mathcal U_Q)\).\looseness=-1

\textbf{Functions.} Let \( \psi_1(v)\!=\!\mathbf 1^\top v \) for \( v\!\in\!\mathbb{R}^p \) and \( \psi_2(t)\!=\!t \) for \( t\!\in\!\mathbb{R} \).
We define $\gA$ and $\gB$ as follows.\looseness=-1

\emph{(i) A structured channel that improves with \(k\).}
Define the structured class \( \mathcal{A} \) by diagonal linear maps:\looseness=-1
\[
\mathcal{A}\coloneqq \Bigl\{\, a(u)=\mathrm{diag}(w_1,\dots,w_p)\,u \ \big|\ 0\le w_i\le\tau\ \text{for all }i \Bigr\},
\]
where \(\tau\!>\!1\). \( \mathcal{A} \) represents a gradually improving structured mechanism: it is well-aligned with the target \(z\) (both are linear maps), but may be unable to represent \(z\) accurately under a small $k$. As \(k\) increases, repeated composition increases  \(\tau^k\) and enables increasingly accurate approximation of \( z \). \looseness=-1

\emph{(ii) An overly expressive shortcut channel.} Let \(\mathcal B\) be the RKHS on \(\mathcal U\) induced by the Kronecker delta kernel \(K(u,u')=\mathbf 1\{u=u'\}\). This shortcut channel can produce essentially arbitrary labels on the finite support of \(\mathcal P\) but provides no meaningful control on unseen inputs from \(\mathcal Q\) when the supports are disjoint. It can be viewed as an analogue of the unconstrained capacity of a heavily overparameterized neural network: for instance, when trained on math data, an LLM may internalize a task-specific computation and produce answers with little step-by-step symbolic reasoning. This can achieve low ID error but need not transfer reliably to unseen problems. We observe related behaviors empirically in Section~\ref{subsec: rl}.\looseness=-1

\textbf{Regularization.}
We regularize the next-token generator \(N\) (Eq.~\ref{eq:next_token_ab}) by its squared functional norm. Concretely, for any induced predictor
\(f_k\!=\!g\!+\!h\) of the form in Eq.~\ref{eq:fk_decompose_ab}, we define\looseness=-1
\[
R(g,h)\;\coloneqq\;\lambda\,\|N\|^2,
\qquad
\|N\|^2\coloneqq \|a\|^2+\|b\|_{\mathcal B}^2,
\]
for some \(\lambda\!>\!0\). \(\|a\|^2\) is the squared Frobenius norm of the induced linear map, i.e., \(\sum_{i=1}^p w_i^2\), and \(\|b\|_{\mathcal B}^2\) is the squared RKHS norm under the kernel
defining \(\mathcal B\). Equivalently, this corresponds to \(R_G(g)\!\coloneqq\! \lambda\|a\|^2\) and
\(R_H(h)\!\coloneqq\! \lambda\|b\|_{\mathcal B}^2\).\looseness=-1

Finally, we verify that the above satisfies Assumption~\ref{as:decomp}:
\begin{proposition}
\label{prop:decomp_instantiation}
Letting \(M\coloneqq 2^{p-1}\), Assumption.~\ref{as:decomp} holds with \(\gamma\!=\!1\), \(C\!=\!1\),
\(a_{\mathcal P}\!=\!a_{\mathcal Q}\!=\!1/(\lambda M)\), \(\rho_k\!=\!\lambda \sum_{i=1}^p \theta_i^{2/k}\),
\(\epsilon_k\!=\!\sum_{i=1}^p ((\theta_i\!-\!\tau^k)_+)^2\), and
\(\delta_{k}\!=\!\frac{\lambda M}{1+\lambda M}\,\|\theta\|_2^2\).
\end{proposition}

To illustrate the theorem numerically, let's take \(p\!=\!8\), \(\theta\!=\![0,2,0, 0\dots]^\top\), \(\tau\!=\!1.1\),  and \(\lambda\!=\!10^{-5}\). Then $\rho_k\!=\!\frac{\sqrt[k]{4}}{10^5}$, $\epsilon_k\!=\!((2\!-\!1.1^k)_+)^2$, $\delta_{k}\!\approx\! 0.005$. By Theorem~\ref{thm:main}, the ID loss satisfies $
L_{\mathcal P}(\hat f)\le \delta_k \leq 0.006$ for all k,
so the ID error is small and plateaus right from the beginning. In contrast, the OOD loss is bounded as $((2\!-\!1.1^k)_+)^2\!\le\!
L_{\mathcal Q}(\hat f)\!\le\!
4692\Bigl(((2-1.1^k)_+)^2\!+\!\frac{\sqrt[k]{4}}{10^5}\Bigr)$.
In particular, when \(k\!=\!1\), \(L_{\mathcal Q}(\hat f)\!\geq\! 0.81\).
When \(k\!\geq\!8\), the upper bound gives \(L_{\mathcal Q}(\hat f)\!\le\! 0.056\).
Thus increasing \(k\) from \(1\) to \(8\) provably reduces the OOD loss by at least \(0.754\),
despite almost no ID improvement.\looseness=-1

\section{Experiments}\label{sec:exp}

\begin{figure}[!t]
    \centering
    \includegraphics[width=0.52\linewidth]{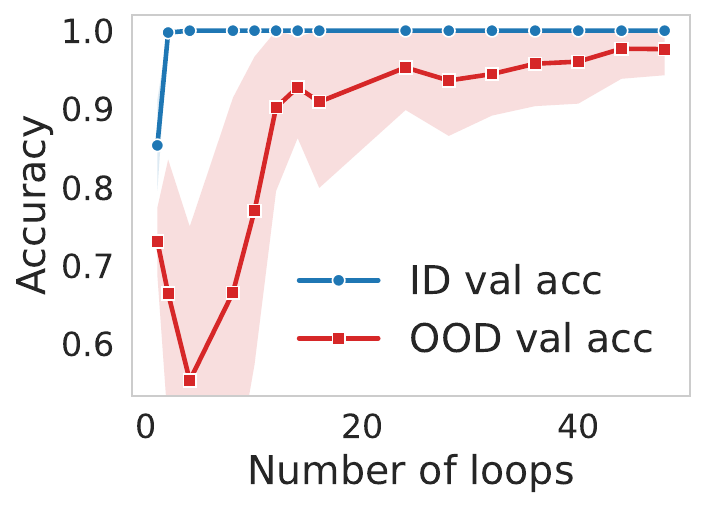}
    \vspace{-.12cm}
    \caption{Looped Transformers trained on $4$-hop problems. Results averaged over 8 random seeds, with shaded regions showing std. Increasing the loop count yields OOD gains up to around 44 loops, while ID accuracy saturates near $100\%$ by around 2 loops.\looseness=-1}
    \label{fig:phop}
    \vspace{-.15cm}
\end{figure}

\subsection{Scaling Latent Reasoning via Looped Transformers on Synthetic Algorithmic Tasks}\label{subsec: phop}

Looped Transformers have recently emerged as a new way to scale reasoning length via additional ``latent thoughts'' \cite{yang2023looped,giannou2023looped,saunshi2025reasoning}. Conceptually, a $k$-loop model can be viewed as the $k$-fold self-composition of a single Transformer block: the same backbone (with shared weights) is applied repeatedly, rather than stacking $k$ distinct layers with different parameters. See the model and training details in Appendix~\ref{apdx:phop_model_training}.
\looseness=-1

\textbf{$p$-hop induction.} We consider the problem studied in \cite{sanford2024transformers,saunshi2025reasoning}. In this task, an input is a sequence over a finite alphabet. Starting from the last position, the model repeatedly ``backtracks'' by jumping to the position after the most recent earlier occurrence of the current symbol; after $p$ such steps, it must output the symbol at the resulting location. See Appendix \ref{apdx:phop_def} for the formal definition. This probes a form of multi-step retrieval that is reminiscent of the iterative reasoning required in certain reading-comprehension-style settings.\looseness=-1

\textbf{ID/OOD split.} Let $i_1,\dots,i_p$ denote the hop indices (i.e., the positions visited during backtracking). We assign an instance to ID if $\sum_{t=1}^p i_t$ is even, and to OOD otherwise.\looseness=-1

Figure~\ref{fig:phop} summarizes the results. As we increase the number of loops, ID validation accuracy quickly saturates to nearly $100\%$ (around 2 loops), while OOD validation accuracy improves much more slowly and continues to increase even at very large loop counts.

\subsection{Scaling Token Budgets in RL on Math Tasks}\label{subsec: rl}

\begin{figure}[!t]
    \centering
    \begin{subfigure}[t]{0.54\linewidth}
        \centering
        \includegraphics[width=.95\linewidth]{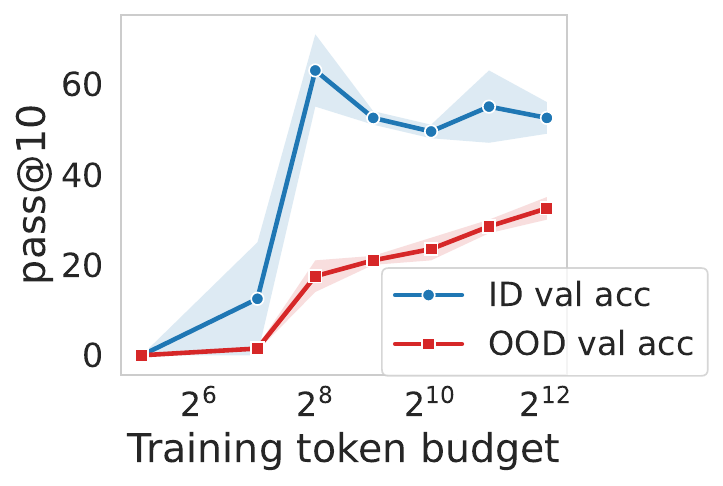}
        \vspace{-.1cm}
        \caption{Eval budget: same as train}
        \label{fig:rl_poly_left}
    \end{subfigure}\hfill
    \begin{subfigure}[t]{0.46\linewidth}
        \centering
        \includegraphics[width=.95\linewidth]{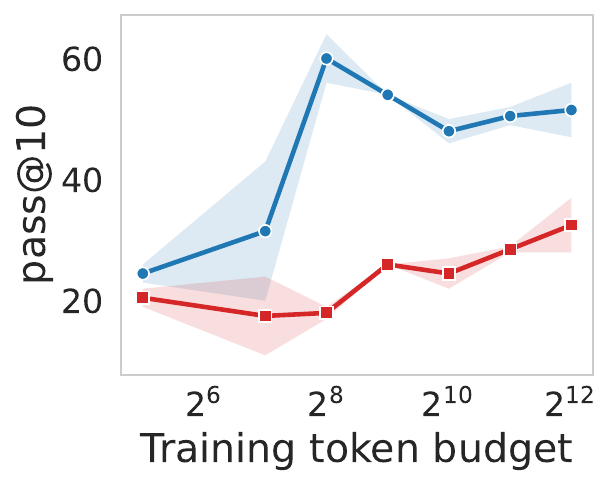}
        \vspace{-.1cm}
        \caption{Eval budget: 4096}
        \label{fig:rl_poly_right}
    \end{subfigure}
    \caption{RL fine-tuning of \texttt{Qwen2.5-1.5B-Instruct} on math data with varying training-time token budgets. Results are averaged over 3 random seeds. ID accuracy peaks at 256 and slightly drops thereafter, while OOD accuracy steadily improves throughout. This trend persists even when fixing the evaluation-time budget to 4096 (b), indicating that the effect is truly driven by the training-time budget. \looseness=-1}
    \vspace{-0.15cm}
    \label{fig:rl_poly}
\end{figure}

\begin{figure}[!t]
    \centering
    \includegraphics[width=1.0\linewidth]{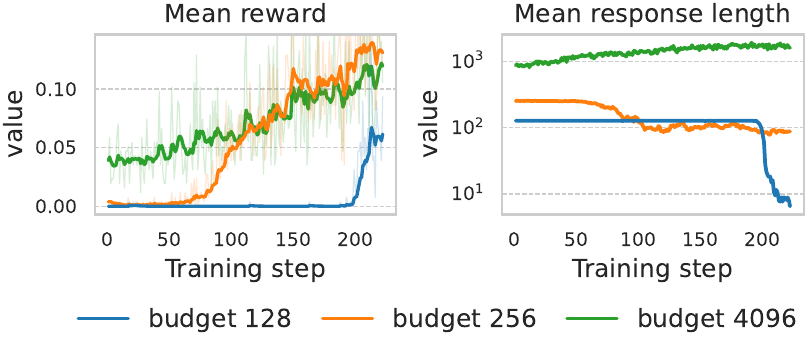}
    \vspace{-.1cm}
    \caption{For small budgets (128/256), the mean reward (left) remains near 0 early in training and rises sharply when the mean response length (right) abruptly drops, indicating a transition to shorter, less step-by-step outputs. In contrast, with a large budget (4096), response length increases throughout training.\looseness=-1}
    \vspace{-.1cm}
    \label{fig:diagnosis}
\end{figure}

We fine-tune reasoning-capable LLMs on math data with RL, varying the training-time token budget, which caps the maximum allowed CoT reasoning length.

\textbf{Dataset.} The ID dataset is an adaptation of the \texttt{Polynomial Roots} dataset from \cite{sun2025omega}, which involves solving polynomial equations. We use our own variant because the original questions are too difficult for the 1.5B model we consider; see Appendix \ref{apdx:rl} for details. The OOD dataset is \texttt{Mixed Operation} also from \cite{sun2025omega}, which asks the model to compute the value of an expression with mixed arithmetic operations. We train on the training split of the ID data and evaluate on the validation splits of both the ID and OOD data.\looseness=-1

\textbf{Model.} We use GRPO \cite{guo2025deepseek} to finetune \texttt{Qwen2.5-1.5B-Instruct} \cite{qwen2.5} with correctness-based rewards and the LoRA optimizer \cite{hu2022lora}. Additional details are in Appendix \ref{apdx:rl}.\looseness=-1

We evaluate the resulting models under two evaluation setups to separate the effect of the training-time token budget from the (potentially confounding) effect of the evaluation-time token budget. In the first (Figure \ref{fig:rl_poly_left}), we set the evaluation-time token budget to match the budget used during training. In the second (Figure \ref{fig:rl_poly_right}), we fix the evaluation-time token budget to 4096 regardless of the training budget. This second setup serves as a control: even when all models are allowed to generate long outputs at test time, we observe the same qualitative trends, indicating that the behavior is primarily driven by the training-time budget rather than a limitation on test-time generation length. Figure~\ref{fig:rl_poly} reports Pass@10 results: under both evaluation setups, ID validation accuracy shows no clear improvement beyond a token budget of 256 and even slightly drops thereafter, while OOD validation accuracy steadily improves throughout, consistent with our phenomenon. Pass@1 results are deferred to Appendix \ref{apdx:rl_additional} and show the same qualitative pattern, except that ID performance oscillates slightly more, while still exhibiting no clear improvement beyond a token budget of 256.\looseness=-1

We also observe an interesting diagnostic pattern in Fig.~\ref{fig:diagnosis}: for small budgets (128 and 256), the mean reward stays near 0 for tens to hundreds of epochs during training, and then begins to increase steadily. This transition coincides almost exactly with a sudden decrease in mean response length, after which the model eventually learns to output answers with little intermediate reasoning. In contrast, for a large budget (e.g., 4096), the mean response length increases throughout training. Combined with the fact that a budget of 256 performs comparably—or even slightly better—than 4096 on ID data but falls short OOD, these observations suggest that the 256-budget model attains similar ID performance via a less generalizable, less step-by-step mechanism. See example model outputs in Appendix~\ref{apdx:rl_additional}.\looseness=-1
\section{Conclusion and Discussion}
We study how scaling training-time (explicit or latent) reasoning length affects OOD generalization under outcome supervision. Viewing reasoning as self-iteration, we show theoretically that increasing it can improve OOD performance even after ID saturation, by reshaping the induced hypothesis class and interacting with regularization to suppress shortcuts. Empirically, we observe this effect in both looped Transformers (increasing loop counts) and RL fine-tuning (increasing token budgets). Our results offer a learning-theoretic understanding of the role of reasoning length and practical guidance for scaling it under distribution shift.\looseness=-1

We note that our theory and experiments do not rule out “seemingly natural” shortcuts that persist during self-iteration yet are not strongly discouraged by the regularizer. In such cases, increasing training-time reasoning length may not yield OOD improvements. The goal of this work is not to claim universality, but to provide evidence that the phenomenon can occur under certain natural conditions.


\section*{Acknowledgement}
This research was supported in part by in part by the NSF CAREER Award 2146492, NSF-Simons AI Institute for Cosmic Origins (CosmicAI), and NSF AI Institute for Foundations of Machine Learning (IFML).

\bibliography{reference}
\bibliographystyle{icml2026}

\newpage
\appendix
\onecolumn

\section{Proof of Theorem~\ref{thm:high_density_unique_ood_clean}}
\label{apdx: extension_cycle}

\subsection{Overview}

We study the ``cycle task'' on a finite set \(\sS\) of size \(n\), where the ground-truth target
\(z:\sS\to\sS\) is a cyclic permutation (a single directed \(n\)-cycle). The training distribution is uniform
on all but one held-out point \(q\in\sS\). We consider the hypothesis class consisting of \(k\)-fold self-iterates
\(f^{\circ k}\) of arbitrary functions \(f:\sS\to\sS\).

The main goal is a \emph{high-density} uniqueness statement: for any desired accuracy level \(\varepsilon>0\),
we will choose a suitable (large) prime \(n\) so that for at least a \((1-\varepsilon)\) fraction of all integers
\(k\geq n-1\), the empirical risk minimizer over \(\gF_k=\gF^{\circ k}\) is \emph{unique} and equals the true cycle map \(z\).

\paragraph{A WLOG reduction to the canonical cycle.}
Because \(z\) is a single \(n\)-cycle, there exists a relabeling bijection \(\pi:\sS\to\sS\) such that
\(\pi\circ z\circ \pi^{-1}=z_0\), where \(z_0(i)=i+1\pmod n\).
Conjugation preserves \(k\)-fold iterates: \((\pi f\pi^{-1})^{\circ k}=\pi f^{\circ k}\pi^{-1}\), and it preserves
which hypotheses interpolate \(\sP\) (up to relabeling the held-out point).
Thus it suffices to prove the theorem for the canonical target \(z(i)=i+1\pmod n\).
We do so below, and the general statement follows by relabeling.

Henceforth, assume
\[
z(i)\coloneqq i+1 \pmod n.
\]

\paragraph{Natural density.}
For a set \(A\subseteq\mathbb N\), define its (lower) natural density by
\[
\dens(A)\coloneqq \liminf_{N\to\infty}\frac{|A\cap\{1,2,\dots,N\}|}{N}.
\]
(All sets we use below are periodic modulo some integer, hence in fact have a true limit density.)
When we say ``a \((1-\varepsilon)\) fraction of integers \(k\geq m\)'', we mean that
\[
\liminf_{N\to\infty}\frac{|\{k\in\{m,\dots,N\}:k\in A\}|}{N-m+1}\ge 1-\varepsilon,
\]
equivalently \(\dens(A)\ge 1-\varepsilon\) since removing finitely many integers does not change density.

\subsection{A structural uniqueness criterion for fixed \(n\) and \(k\)}

\subsubsection{Functional-graph facts for large iterates}

\paragraph{Functional-graph terminology.}
Any map \(g:\sS\to\sS\) can be viewed as a directed graph with outdegree \(1\) at each vertex.
A point \(x\) is \emph{periodic for \(g\)} if \(g^{\circ t}(x)=x\) for some \(t\ge 1\).
The \emph{height} \(\mathrm{ht}(g)\) is the maximum, over all \(x\in\sS\), of the length of the shortest directed path from \(x\) to the set of periodic points of \(g\).
Equivalently,
\[
\mathrm{ht}(g)\le 1
\quad\Longleftrightarrow\quad
\text{for every }x\in\sS,\ \ g(x)\text{ is periodic for }g.
\]

\begin{lemma}[Large iterates collapse height]
\label{lem:height-collapse}
Let \(|\sS|=n\) and let \(k\ge n-1\). For every function \(f:\sS\to\sS\), the iterate
\(g\coloneqq f^{\circ k}\) satisfies \(\mathrm{ht}(g)\le 1\).
\end{lemma}

\begin{proof}
Fix \(x\in\sS\) and consider the forward orbit under \(f\):
\[
x,\ f(x),\ f^{\circ 2}(x),\ \dots.
\]
Among the \(n+1\) points \(x,f(x),\dots,f^{\circ n}(x)\) there must be a repetition, so there exist indices
\(0\le i<j\le n\) such that \(f^{\circ i}(x)=f^{\circ j}(x)\). Let \(y\coloneqq f^{\circ i}(x)\).
Then \(f^{\circ (j-i)}(y)=y\), so \(y\) is periodic for \(f\), and moreover \(i\le n-1\).

Since \(k\ge n-1\ge i\), we may write
\[
g(x)=f^{\circ k}(x)=f^{\circ (k-i)}(y).
\]
Because \(y\) lies on a directed cycle of \(f\), every forward iterate of \(y\) also lies on that cycle; in particular \(g(x)\) is periodic for \(f\).

Finally, if a point \(u\) is periodic for \(f\), say \(f^{\circ m}(u)=u\), then it is also periodic for
\(g=f^{\circ k}\) because
\[
g^{\circ m}(u)=(f^{\circ k})^{\circ m}(u)=f^{\circ (km)}(u)=u.
\]
Applying this to \(u=g(x)\) shows that \(g(x)\) is periodic for \(g\).
Since \(x\) was arbitrary, every point reaches a periodic point of \(g\) in at most one step, i.e.\ \(\mathrm{ht}(g)\le 1\).
\end{proof}

\subsubsection{The only height-\(\le 1\) interpolants}

Define the ``skip-one'' map \(\tilde z:\sS\to\sS\) by
\begin{equation}
\label{eq:tilde-z}
\tilde z(x)\coloneqq
\begin{cases}
z(x), & x\neq q,\\
z^{\circ 2}(q), & x=q.
\end{cases}
\end{equation}
Thus \(\tilde z\) agrees with \(z\) on the training set \(\sP\) but sends the held-out point \(q\) forward by two steps.

\begin{lemma}[Only two interpolants can have height at most \(1\)]
\label{lem:only-two-interpolants}
Let \(g:\sS\to\sS\) satisfy \(g(x)=z(x)\) for all \(x\in\sP\).
If \(\mathrm{ht}(g)\le 1\), then \(g\) must equal either \(z\) or \(\tilde z\).
\end{lemma}

\begin{proof}
Because \(z\) is an \(n\)-cycle, there is a unique integer \(t\in\{1,2,\dots,n\}\) such that
\[
g(q)=z^{\circ t}(q),
\]
where \(t=n\) corresponds to \(g(q)=q\).

If \(t=1\), then \(g(q)=z(q)\) and hence \(g=z\).

If \(t=2\), then \(g\) agrees with \(z\) on \(\sP\) and sends \(q\) to \(z^{\circ 2}(q)\), so \(g=\tilde z\).

It remains to show that \(t\ge 3\) is impossible under \(\mathrm{ht}(g)\le 1\).

Let \(r\coloneqq z(q)\in\sP\). Then \(g(r)=z(r)=z^{\circ 2}(q)\).
We claim that \(g(r)\) is \emph{not} periodic for \(g\) when \(t\ge 3\).

Indeed, for each \(j\in\{0,1,\dots,t-3\}\), the point \(z^{\circ (j+2)}(q)\) is not equal to \(q\), so along this segment \(g\) acts exactly as \(z\):
\[
g^{\circ j}(g(r))=g^{\circ j}(z^{\circ 2}(q))=z^{\circ (j+2)}(q).
\]
In particular, taking \(j=t-3\) gives \(g^{\circ (t-3)}(g(r))=z^{\circ (t-1)}(q)\).
Since \(t\in\{3,\dots,n\}\) implies \(z^{\circ (t-1)}(q)\neq q\), one more application of \(g\) still follows \(z\), hence
\[
g^{\circ (t-2)}(g(r)) = g\!\bigl(z^{\circ (t-1)}(q)\bigr)=z^{\circ t}(q)=g(q).
\]

The point \(g(q)\) lies on the unique directed cycle of \(g\):
\begin{itemize}
\item If \(t<n\), then \(q\to z^{\circ t}(q)\to z^{\circ (t+1)}(q)\to\cdots\to z^{\circ (n-1)}(q)\to q\) is a directed cycle.
\item If \(t=n\), then \(g(q)=q\) and \(q\) is a fixed point.
\end{itemize}
When \(t\ge 3\), the point \(g(r)=z^{\circ 2}(q)\) is not on that cycle; it flows into the cycle after at least one step and therefore cannot be periodic.

Thus there exists \(x=r\) such that \(g(x)=g(r)\) is \emph{not} periodic, which contradicts \(\mathrm{ht}(g)\le 1\).
Therefore \(t\not\ge 3\), and the only possibilities are \(t\in\{1,2\}\), giving \(g\in\{z,\tilde z\}\).
\end{proof}

\subsubsection{Which interpolants are \(k\)-fold iterates?}

\begin{lemma}[\(z\) is a \(k\)-fold iterate iff \(\gcd(k,n)=1\)]
\label{lem:z-is-kth-iterate}
Let \(|\sS|=n\) and \(k\ge 1\). Then
\[
z\in \gF_k \quad\Longleftrightarrow\quad \gcd(k,n)=1.
\]
\end{lemma}

\begin{proof}
(\(\Rightarrow\)) Suppose \(z=f^{\circ k}\) for some \(f:\sS\to\sS\).
Since \(z\) is bijective, \(f^{\circ k}\) is bijective, hence \(f\) is injective: if \(f(x)=f(y)\) then applying \(f^{\circ (k-1)}\) gives
\(f^{\circ k}(x)=f^{\circ k}(y)\), and bijectivity of \(f^{\circ k}\) implies \(x=y\).
On a finite set, injective implies bijective, so \(f\) is a permutation of \(\sS\).

Moreover,
\[
f\circ z = f\circ f^{\circ k} = f^{\circ (k+1)} = f^{\circ k}\circ f = z\circ f,
\]
so \(f\) commutes with \(z\).
Let \(a\in\{0,1,\dots,n-1\}\) be such that \(f(0)=a=z^{\circ a}(0)\).
Then for every \(i\in\sS\),
\[
f(i)=f(z^{\circ i}(0))=z^{\circ i}(f(0))=z^{\circ i}(a)=z^{\circ a}(i),
\]
so \(f=z^{\circ a}\).
Now \(f^{\circ k}=z^{\circ (ak)}=z\) implies \(ak\equiv 1\pmod n\), which is solvable if and only if \(\gcd(k,n)=1\).

(\(\Leftarrow\)) If \(\gcd(k,n)=1\), let \(a\) be the inverse of \(k\) modulo \(n\), i.e.\ \(ak\equiv 1\pmod n\).
Set \(f\coloneqq z^{\circ a}\). Then
\[
f^{\circ k}=z^{\circ (ak)}=z^{\circ 1}=z,
\]
so \(z\in\gF_k\).
\end{proof}

\begin{lemma}[\(\tilde z\) is a \(k\)-fold iterate iff \(\gcd(k,n-1)=1\)]
\label{lem:tilde-is-kth-iterate}
Let \(\tilde z\) be defined by~\eqref{eq:tilde-z} and set \(L\coloneqq n-1\).
Then for any \(k\ge 1\),
\[
\tilde z\in \gF_k
\quad\Longleftrightarrow\quad
\gcd(k,L)=1.
\]
\end{lemma}

\begin{proof}
Let \(r\coloneqq z(q)\) and let \(C\coloneqq \sS\setminus\{r\}\), so \(|C|=L=n-1\).
By construction, \(\tilde z\) restricts to a single directed cycle on \(C\) of length \(L\) (it is the \(z\)-cycle with the point \(r\) removed), and \(r\) maps into that cycle:
\[
\tilde z(r)=z^{\circ 2}(q)\in C.
\]
In particular, the periodic points of \(\tilde z\) are exactly the points of \(C\).

\smallskip
\noindent\emph{Fact: periodic points are preserved under iterates.}
For any map \(f\) and any \(k\ge 1\), a point \(x\) is periodic for \(f\) if and only if it is periodic for \(f^{\circ k}\):
\[
f^{\circ m}(x)=x \ \Rightarrow\ (f^{\circ k})^{\circ m}(x)=f^{\circ (km)}(x)=x,
\]
and conversely
\[
(f^{\circ k})^{\circ m}(x)=x\ \Rightarrow\ f^{\circ (km)}(x)=x.
\]
Thus \(f\) and \(f^{\circ k}\) have exactly the same periodic points.

\smallskip
\noindent\emph{(\(\Rightarrow\)) Necessity.}
Assume \(\tilde z=f^{\circ k}\) for some \(f\).
By the fact above, \(f\) and \(\tilde z\) have the same periodic points, so the periodic points of \(f\) are exactly \(C\).
In particular \(f(C)\subseteq C\) and every point of \(C\) is periodic for \(f\), so \(f|_C\) is a permutation of \(C\).

Since \(\tilde z|_C=(f|_C)^{\circ k}\) is a \emph{single} cycle of length \(L\), the permutation \(f|_C\) must itself be a single cycle of length \(L\)
(taking powers of a permutation cannot merge distinct cycles).
On a cycle of length \(L\), taking the \(k\)th power splits the cycle into \(\gcd(L,k)\) disjoint cycles.
Because \((f|_C)^{\circ k}=\tilde z|_C\) is a single cycle, we must have \(\gcd(L,k)=1\).

\smallskip
\noindent\emph{(\(\Leftarrow\)) Sufficiency.}
Assume \(\gcd(k,L)=1\).
Enumerate the cycle \(C\) in the order induced by \(\tilde z\) as
\[
e_0\coloneqq q,\quad e_1\coloneqq z^{\circ 2}(q),\quad e_2\coloneqq z^{\circ 3}(q),\ \dots,\quad e_{L-1}\coloneqq z^{\circ (n-1)}(q),
\]
so that \(\tilde z(e_i)=e_{i+1\ \mathrm{mod}\ L}\) for all \(i\).

Let \(a\) be the inverse of \(k\) modulo \(L\), i.e.\ \(ak\equiv 1\pmod L\), and define \(f\) on \(C\) by
\[
f(e_i)\coloneqq e_{i+a\ \mathrm{mod}\ L}.
\]
Then for each \(e_i\in C\),
\[
f^{\circ k}(e_i)=e_{i+ak\ \mathrm{mod}\ L}=e_{i+1\ \mathrm{mod}\ L}=\tilde z(e_i).
\]

It remains to define \(f\) on the remaining point \(r\notin C\).
We want \(f^{\circ k}(r)=\tilde z(r)=e_1\).
Since \(f(r)\) will lie in \(C\) and \(f\) acts on \(C\) by \(e_i\mapsto e_{i+a}\), it suffices to choose \(f(r)=e_j\) so that
\[
f^{\circ (k-1)}(e_j)=e_1.
\]
But \(f^{\circ (k-1)}(e_j)=e_{j+a(k-1)\ \mathrm{mod}\ L}\), so we need
\[
j+a(k-1)\equiv 1 \pmod L.
\]
There is a unique \(j\in\{0,1,\dots,L-1\}\) satisfying this congruence; fix that \(j\) and set \(f(r)\coloneqq e_j\).
Then
\[
f^{\circ k}(r)=f^{\circ (k-1)}(f(r))=f^{\circ (k-1)}(e_j)=e_{j+a(k-1)}=e_1=\tilde z(r).
\]
Thus \(f^{\circ k}=\tilde z\) on all of \(\sS\), proving \(\tilde z\in\gF_k\).
\end{proof}

\subsubsection{Uniqueness for fixed \(n,k\)}

\begin{theorem}[Uniqueness of the good interpolant for large iterates]
\label{thm:unique-good}
Fix \(n\ge 3\) and a held-out point \(q\in\sS\).
Let \(k\ge n-1\) satisfy
\[
\gcd(k,n)=1
\quad\text{and}\quad
\gcd(k,n-1)>1.
\]
Then
\[
\arg\min_{g\in \gF_k} L_{\gP}(g)=\{z\},
\]
and consequently the unique minimizer \(\hat g=z\) has perfect held-out performance:
\[
L_{\gQ}(\hat g)=0.
\]
\end{theorem}

\begin{proof}
Because \(\gcd(k,n)=1\), Lemma~\ref{lem:z-is-kth-iterate} implies \(z\in\gF_k\), hence the minimum possible training loss over \(\gF_k\)
is at most \(L_{\gP}(z)=0\) and therefore equals \(0\).

Now let \(g\in\gF_k\) satisfy \(L_{\gP}(g)=0\).
Then \(g(x)=z(x)\) for all \(x\in\sP\).
Since \(g\in\gF_k\) and \(k\ge n-1\), Lemma~\ref{lem:height-collapse} gives \(\mathrm{ht}(g)\le 1\).
Lemma~\ref{lem:only-two-interpolants} then forces \(g\in\{z,\tilde z\}\).

Finally, because \(\gcd(k,n-1)>1\), Lemma~\ref{lem:tilde-is-kth-iterate} implies \(\tilde z\notin\gF_k\).
Thus the only hypothesis in \(\gF_k\) with zero training loss is \(z\), so the minimizer is unique and equals \(z\).
Since \(\gQ\) is the point mass at \(q\), we have \(L_{\gQ}(z)=\mathbf{1}\{z(q)\neq z(q)\}=0\).
\end{proof}

\subsection{Number theory: densities, totients, and Dirichlet}

\begin{lemma}[Densities of multiples and coprimes]
\label{lem:density-basic}
Let \(m\ge 1\) be an integer.
\begin{enumerate}
\item The set \(\{k\in\mathbb N: m\mid k\}\) has density \(1/m\).
\item The set \(\{k\in\mathbb N:\gcd(k,m)=1\}\) has density \(\varphi(m)/m\), where \(\varphi\) is Euler's totient function.
\end{enumerate}
\end{lemma}

\begin{proof}
Both statements follow from periodicity modulo \(m\).
For (1), in each block of \(m\) consecutive integers there is exactly one multiple of \(m\).
For (2), among the residues \(\{1,2,\dots,m\}\) there are exactly \(\varphi(m)\) residues that are coprime to \(m\),
and each block of length \(m\) repeats the same pattern.
\end{proof}

\begin{lemma}[Divisibility can only decrease \(\varphi(\cdot)/(\cdot)\)]
\label{lem:phi-ratio-monotone}
If \(M\mid N\), then
\[
\frac{\varphi(N)}{N}\le \frac{\varphi(M)}{M}.
\]
\end{lemma}

\begin{proof}
Using the prime-factor formula,
\[
\frac{\varphi(T)}{T}=\prod_{p\mid T}\left(1-\frac1p\right).
\]
If \(M\mid N\), then every prime divisor of \(M\) is also a prime divisor of \(N\), so the product defining \(\varphi(N)/N\)
contains all factors from \(\varphi(M)/M\), and possibly additional factors of the form \((1-1/p)\le 1\).
Hence \(\varphi(N)/N\le \varphi(M)/M\).
\end{proof}

\begin{lemma}[Making \(\varphi(M)/M\) arbitrarily small]
\label{lem:small-phi-ratio}
For every \(\delta>0\) there exists an integer \(M\ge 2\) such that
\[
\frac{\varphi(M)}{M}<\delta.
\]
\end{lemma}

\begin{proof}
Let \(p_1<p_2<\cdots\) be the primes and define the primorial
\[
M_t \coloneqq \prod_{i=1}^t p_i.
\]
Then
\[
\frac{\varphi(M_t)}{M_t}=\prod_{i=1}^t \left(1-\frac1{p_i}\right).
\]
Using \(\log(1-x)\le -x\) for \(x\in(0,1)\),
\[
\log\!\left(\prod_{i=1}^t \left(1-\frac1{p_i}\right)\right)
=\sum_{i=1}^t \log\!\left(1-\frac1{p_i}\right)
\le -\sum_{i=1}^t \frac1{p_i}.
\]
A classical theorem of Euler states that \(\sum_{p\ \mathrm{prime}} \frac1p\) diverges, so the RHS tends to \(-\infty\),
forcing \(\prod_{i=1}^t (1-1/p_i)\to 0\). Hence for large enough \(t\), \(\varphi(M_t)/M_t<\delta\); set \(M\coloneqq M_t\).
\end{proof}

\begin{theorem}[Dirichlet's theorem on arithmetic progressions]
\label{thm:dirichlet}
If \(a,m\in\mathbb N\) satisfy \(\gcd(a,m)=1\), then there exist infinitely many primes \(p\) such that \(p\equiv a\pmod m\).
\end{theorem}

\begin{lemma}[A density lower bound for ``good'' exponents]
\label{lem:density-good-k}
Let \(n\ge 3\) be prime and define
\[
G_n \coloneqq \{k\in\mathbb N:\gcd(k,n)=1\ \text{and}\ \gcd(k,n-1)>1\}.
\]
Then \(G_n\) has density at least
\[
\dens(G_n)\ \ge\ 1-\frac1n-\frac{\varphi(n-1)}{n-1}.
\]
\end{lemma}

\begin{proof}
Because \(n\) is prime, \(\gcd(k,n)=1\) fails exactly when \(n\mid k\), which has density \(1/n\) by Lemma~\ref{lem:density-basic}(1).
Also \(\gcd(k,n-1)>1\) fails exactly when \(\gcd(k,n-1)=1\), which has density \(\varphi(n-1)/(n-1)\) by Lemma~\ref{lem:density-basic}(2).

Thus the complement \(G_n^c\) is contained in the union
\[
\{k: n\mid k\}\ \cup\ \{k:\gcd(k,n-1)=1\}.
\]
Taking densities and using subadditivity (a union bound) yields
\[
\dens(G_n^c)\ \le\ \frac1n+\frac{\varphi(n-1)}{n-1}.
\]
Therefore \(\dens(G_n)=1-\dens(G_n^c)\ge 1-\frac1n-\frac{\varphi(n-1)}{n-1}\).
\end{proof}

\subsection{High-density uniqueness for the cycle task via large iterates}

Fix \(\varepsilon>0\).
We will choose a prime \(n\) such that
\begin{equation}
\label{eq:choose-n-conditions}
\frac1n<\frac{\varepsilon}{2}
\qquad\text{and}\qquad
\frac{\varphi(n-1)}{n-1}<\frac{\varepsilon}{2}.
\end{equation}
Then Lemma~\ref{lem:density-good-k} will imply \(\dens(G_n)\ge 1-\varepsilon\).
Finally, Theorem~\ref{thm:unique-good} will transfer this density statement to the uniqueness property.

\smallskip
\noindent\textbf{Step 1: choose \(M\) with \(\varphi(M)/M<\varepsilon/2\).}
By Lemma~\ref{lem:small-phi-ratio}, choose an integer \(M\ge 2\) such that
\[
\frac{\varphi(M)}{M}<\frac{\varepsilon}{2}.
\]

\smallskip
\noindent\textbf{Step 2: use Dirichlet to pick a large prime \(n\equiv 1\pmod M\).}
By Dirichlet's theorem (Theorem~\ref{thm:dirichlet}) with \(a=1\) and modulus \(m=M\), there exist infinitely many primes \(n\) with
\[
n\equiv 1\pmod M.
\]
Choose such a prime \(n\) large enough that \(1/n<\varepsilon/2\).
Then \(M\mid (n-1)\).

\smallskip
\noindent\textbf{Step 3: bound \(\varphi(n-1)/(n-1)\).}
Since \(M\mid (n-1)\), Lemma~\ref{lem:phi-ratio-monotone} gives
\[
\frac{\varphi(n-1)}{n-1}\le \frac{\varphi(M)}{M}<\frac{\varepsilon}{2}.
\]
Together with \(1/n<\varepsilon/2\), this establishes~\eqref{eq:choose-n-conditions}.

\smallskip
\noindent\textbf{Step 4: conclude a density-\(\ge 1-\varepsilon\) set of ``good'' exponents.}
Define
\[
G_n = \{k\in\mathbb N:\gcd(k,n)=1\ \text{and}\ \gcd(k,n-1)>1\}.
\]
By Lemma~\ref{lem:density-good-k} and~\eqref{eq:choose-n-conditions},
\[
\dens(G_n)\ge 1-\frac1n-\frac{\varphi(n-1)}{n-1}\ \ge\ 1-\varepsilon.
\]
Since removing finitely many integers does not change density, the same bound holds if we restrict to $k\ge n-1$:
\[
\dens(\{k\ge n-1:\,k\in G_n\})=\dens(G_n)\ge 1-\varepsilon.
\]

\smallskip
\noindent\textbf{Step 5: transfer ``good gcd'' to uniqueness.}
For any \(k\geq n-1\) with \(k\in G_n\), we have \(\gcd(k,n)=1\), and \(\gcd(k,n-1)>1\).
Therefore Theorem~\ref{thm:unique-good} applies and yields
\[
\arg\min_{g\in\gF_k}L_{\gP}(g)=\{z\}.
\]
Consequently the unique minimizer is \(\hat g=z\), and since \(\gQ\) is a point mass at \(q\) we have
\(L_{\gQ}(\hat g)=L_{\gQ}(z)=0\).

Thus for at least a $(1-\varepsilon)$ fraction of integers $k\ge n-1$, the uniqueness property holds.

\section{Proof of Theorem~\ref{thm:two_point_iso}}
\label{sec:proof_two_point_iso}

\subsection{Explicit Form of the Infinite-Context Limit}\label{subsec:infinite_context_explicit}

For each $n$, let $\mS^{(n)}_{\mSigma,\bm\beta}\in\R^{d_e\times(n+1)}$ denote the input sequence constructed from
$n$ samples $\{(\vx_i,y_i)\}_{i=1}^n$ via \eqref{eq:linreg_sequence}. Let $\mX=[\vx_1,\dots,\vx_n]\in\R^{d\times n}$
and $\vy=(y_1,\dots,y_n)^\top\in\R^n$. Define the extracted coefficient vector at
generation step $t$ by
\[
\hat{\vw}^{(n)}_t \;:=\; \Extract\!\left(\Auto^{\circ t}\big(\mS^{(n)}_{\mSigma,\bm\beta}\mid \mV,\mW\big)\right)\in\R^d ,
\]
so that $\hat{\vw}^{(n)}_k = f_k(\mS^{(n)}_{\mSigma,\bm\beta}\mid \mV,\mW)$ by definition of $f_k$.

A direct block-matrix calculation using the sparsity of $\tilde{\mV},\tilde{\mW}$ shows that, for every $t\ge 0$,
\begin{align}
\label{eq:update_general_finite_n_app}
\hat{\vw}^{(n)}_{t+1}
&=
\hat{\vw}^{(n)}_t+\mV\Big(\widehat{\mSigma}^{(n)}\,\mW\,\hat{\vw}^{(n)}_t-\widehat{\mathbf{u}}^{(n)}\Big),
\qquad
\hat{\vw}^{(n)}_0=\mathbf{0},
\end{align}
where
\[
\widehat{\mSigma}^{(n)}:=\frac{1}{n}\sum_{i=1}^n \vx_i\vx_i^\top=\frac{1}{n}\mX\mX^\top,
\qquad
\widehat{\mathbf{u}}^{(n)}:=\frac{1}{n}\sum_{i=1}^n \vx_i y_i=\frac{1}{n}\mX\vy .
\]
In particular, conditioned on the dataset, the induced update depends on the samples only through the empirical moments
$\widehat{\mSigma}^{(n)}$ and $\widehat{\mathbf{u}}^{(n)}$.

Unrolling the recursion \eqref{eq:update_general_finite_n_app} yields the closed form
\begin{align}
\label{eq:finite_n_closed_form}
\hat{\vw}^{(n)}_k
=
-\sum_{t=0}^{k-1}\Big(\mI+\mV\widehat{\mSigma}^{(n)}\mW\Big)^t\,\mV\,\widehat{\mathbf{u}}^{(n)} .
\end{align}
For fixed $(\mV,\mW)$ and $k$, define the deterministic map
\[
g_k(\mSigma,\mathbf{u})
\;:=\;
-\sum_{t=0}^{k-1}\Big(\mI+\mV\mSigma\mW\Big)^t\,\mV\,\mathbf{u},
\qquad
(\mSigma,\mathbf{u})\in\R^{d\times d}\times\R^d.
\]
Then \eqref{eq:finite_n_closed_form} can be written as
\[
\hat{\vw}^{(n)}_k
=
g_k\!\big(\widehat{\mSigma}^{(n)},\widehat{\mathbf{u}}^{(n)}\big).
\]
Since $k$ is fixed, $g_k$ is obtained from its inputs via finitely many additions and multiplications (matrix products and sums in a fixed order);
equivalently, each coordinate of $g_k(\mSigma,\mathbf{u})$ is a multivariate polynomial in the entries of $(\mSigma,\mathbf{u})$.
In particular, $g_k$ is continuous. Therefore, if $(\widehat{\mSigma}^{(n)},\widehat{\mathbf{u}}^{(n)})\to(\mSigma,\mathbf{u})$ almost surely, then
$g_k(\widehat{\mSigma}^{(n)},\widehat{\mathbf{u}}^{(n)})\to g_k(\mSigma,\mathbf{u})$ almost surely (by the Continuous Mapping Theorem).

Under our data model $\vx_i\sim\cN(0,\mSigma)$ and $y_i=\vx_i^\top\bm\beta$, we have almost surely
\[
\widehat{\mSigma}^{(n)}\to \mSigma,
\qquad
\widehat{\mathbf{u}}^{(n)}=\frac{1}{n}\sum_{i=1}^n \vx_i\vx_i^\top\bm\beta \to \mSigma\bm\beta,
\qquad (n\to\infty).
\]
Combining these facts yields the almost sure limit
\begin{align}
\label{eq:population_limit_closed_form}
\lim_{n\to\infty} f_k\big(\mS^{(n)}_{\mSigma,\bm\beta}\mid \mV,\mW\big)
=
g_k(\mSigma,\mSigma\bm\beta)
=
-\sum_{t=0}^{k-1}\Big(\mI+\mV\mSigma\mW\Big)^t\,\mV\,\mSigma\,\bm\beta .
\end{align}
Equivalently, writing $\hat{\vw}_k:=g_k(\mSigma,\mSigma\bm\beta)$ and letting $\mA:=\mI+\mV\mSigma\mW$,
\begin{align}
\label{eq:population_limit_error_form}
\hat{\vw}_k-\bm\beta
=
-\Big(\sum_{t=0}^{k-1}\mA^{t}\,\mV\mSigma+\mI\Big)\bm\beta .
\end{align}

\subsection{Characterizing in-distribution minimizers and convergence to GD} \label{subsec:proof_two_point_iso_main}

We repeatedly invoke the finite-$n$ update \eqref{eq:update_general_finite_n_app} and the population-limit expression
\eqref{eq:population_limit_closed_form}--\eqref{eq:population_limit_error_form} established in
Subsection~\ref{subsec:infinite_context_explicit}.

\paragraph{Step 1: Characterizing generic GD-implementing parameters.}
Recall the definition of $\Theta_{\rm GD}$ from \eqref{eq:theta_gd_def}: we say that $(\mV,\mW)$ implements one-step GD
if there exists $\eta>0$ such that, for every dataset (equivalently, for all realizable empirical moments
$(\widehat{\mSigma},\widehat{\mathbf{u}})$) and every iterate $\vw\in\R^d$, the induced update equals the GD update.

By \eqref{eq:update_general_finite_n_app}, the transformer-induced one-step update has the form
\[
U_{\mV,\mW}(\vw;\widehat{\mSigma},\widehat{\mathbf{u}})
:= \vw + \mV\big(\widehat{\mSigma}\mW\vw-\widehat{\mathbf{u}}\big),
\]
whereas one-step GD on the squared loss is
\[
U_{{\rm GD},\eta}(\vw;\widehat{\mSigma},\widehat{\mathbf{u}})
:= \vw - \eta\big(\widehat{\mSigma}\vw-\widehat{\mathbf{u}}\big).
\]
If $U_{\mV,\mW}\equiv U_{{\rm GD},\eta}$ holds for all $\vw$ and all $(\widehat{\mSigma},\widehat{\mathbf{u}})$, then
comparing the coefficients of $\widehat{\mathbf{u}}$ forces $-\mV=\eta\mI_d$, i.e., $\mV=-\eta\mI_d$.
Substituting this into the coefficient of $\widehat{\mSigma}\vw$ yields
\[
\mV\widehat{\mSigma}\mW=-\eta\widehat{\mSigma}
\quad\text{for all }\widehat{\mSigma},
\]
which implies $\mW=\mI_d$. Hence
\begin{equation}
\label{eq:theta_gd}
\Theta_{\rm GD}
=\Big\{\,(-\eta\mI_d,\ \mI_d)\ \big|\ \eta>0\,\Big\}.
\end{equation}

\paragraph{Step 2: Infinite-context error under isotropic $\mSigma=\gamma \mI_d$.}
We now work in the infinite-context limit. By \eqref{eq:population_limit_error_form}, for any covariance
$\mSigma$ we have
\[
\hat{\vw}_k-\bm\beta
=
-\Big(\sum_{t=0}^{k-1}\mA^{t}\,\mV\mSigma+\mI\Big)\bm\beta,
\qquad
\mA:=\mI+\mV\mSigma\mW.
\]
Specializing to the isotropic case $\mSigma=\gamma\mI_d$ yields
\begin{align}
\label{eq:error_matrix_def_reused}
\hat{\vw}_k-\bm\beta
= -\mA_\gamma(\mV,\mW)\,\bm\beta,
\qquad
\mA_\gamma(\mV,\mW)
:=\sum_{t=0}^{k-1}(\mI+\gamma\mV\mW)^t\,\gamma\mV+\mI.
\end{align}
Since $\bm\beta\sim\gN(0,\mI_d)$, we have $\E_{\bm\beta}\|\mA_\gamma\bm\beta\|^2=\|\mA_\gamma\|_F^2$. Therefore, under
$\gP$ (uniform over $\gamma\in\{\gamma_1,\gamma_2\}$),
\begin{align}
\label{eq:LP_frob_reused}
L_{\gP}(f_k)
=\frac{1}{2}\|\mA_{\gamma_1}(\mV,\mW)\|_F^2+\frac{1}{2}\|\mA_{\gamma_2}(\mV,\mW)\|_F^2.
\end{align}

\paragraph{Step 3: Reduce to a scalar problem via simultaneous orthogonal diagonalizability.}
Recall that we only consider $\mV,\mW$ that are simultaneously orthogonally diagonalizable over $\sR$, i.e., there exists an orthogonal matrix
$\mU$ such that
\[
\mV=\mU\diag(v_1,\dots,v_d)\mU^\top,\qquad
\mW=\mU\diag(w_1,\dots,w_d)\mU^\top.
\]
Since $\mA_\gamma(\mV,\mW)$ is a matrix polynomial in $\mV$ and $\mW$, it is diagonalized by the same $\mU$:
\[
\mA_\gamma(\mV,\mW)=\mU\diag(A_{\gamma,1},\dots,A_{\gamma,d})\mU^\top,
\]
where, for each $i\in[d]$,
\begin{align}
\label{eq:A_gamma_i_reused}
A_{\gamma,i}
=
1+\gamma v_i\sum_{t=0}^{k-1}(1+\gamma v_i w_i)^t
=
1+\frac{(1+\gamma v_i w_i)^k-1}{w_i}.
\end{align}
The last equation is for $w_i\neq 0$. Note that, if $w_i=0$, then $A_{\gamma,i}=1+\gamma v_i\sum_{t=0}^{k-1}1=1+k\gamma v_i$, which cannot vanish simultaneously for
$\gamma=\gamma_1$ and $\gamma=\gamma_2$ since $\gamma_1\neq\gamma_2$. Hence any zero-loss solution must satisfy $w_i\neq 0$. Because $\mU$ is orthogonal, $\|\mA_\gamma\|_F^2=\sum_{i=1}^d A_{\gamma,i}^2$, and \eqref{eq:LP_frob_reused} becomes
\begin{align}
\label{eq:LP_scalar_sum_reused}
L_{\gP}(f_k)
=\frac{1}{2}\sum_{i=1}^d A_{\gamma_1,i}^2+\frac{1}{2}\sum_{i=1}^d A_{\gamma_2,i}^2,
\end{align}
so the minimization decouples across coordinates.

\paragraph{Step 4: Solve the per-coordinate minimization problem (even $k$).}
Fix $i$ and write $v=v_i$, $w=w_i$. Achieving $L_{\gP}(f_k)=0$ requires $A_{\gamma_1}(v,w)=A_{\gamma_2}(v,w)=0$,
which is equivalent to
\[
(1+\gamma v w)^k = 1-w
\quad\text{for }\gamma\in\{\gamma_1,\gamma_2\}.
\]
In particular, $(1+\gamma_1 v w)^k=(1+\gamma_2 v w)^k$. Let $x:=1+\gamma_1 v w$ and $y:=1+\gamma_2 v w$.
Since $k$ is even and $x,y\in\R$, the equality $x^k=y^k$ implies $|x|=|y|$, hence $x=\pm y$.

If $x=y$, then $(\gamma_1-\gamma_2)vw=0$. Since $\gamma_1\neq\gamma_2$, this forces $vw=0$, and substituting into
$(1+\gamma v w)^k=1-w$ gives $1=1-w$, i.e., $w=0$. As noted in Step~3, the case $w=0$ cannot yield zero loss at both
$\gamma_1$ and $\gamma_2$. Therefore the case $x=y$ is impossible, and we must have $x=-y$, i.e.,
\[
1+\gamma_1 v w = -(1+\gamma_2 v w).
\]
Solving gives $vw=-1/a$, where $a:=(\gamma_1+\gamma_2)/2$. Writing $b:=|\gamma_1-\gamma_2|/2$ and using
$\gamma_1=a+b$ and $\gamma_2=a-b$ then yields
\[
w = 1-\left(\frac{b}{a}\right)^k,
\qquad
v = -\frac{1}{a\left(1-(b/a)^k\right)}.
\]
Thus the pair $(v_i,w_i)$ is uniquely determined and identical for all $i$, and consequently every in-distribution minimizer
$f_k^\star$ must have parameters
\begin{align}
\label{eq:VW_star_closed_reused}
\mW^\star=\Big(1-(b/a)^k\Big)\mI_d,
\qquad
\mV^\star=-\frac{1}{a\left(1-(b/a)^k\right)}\,\mI_d,
\end{align}
which proves item~(1).

\paragraph{Step 5: Compute the OOD loss and monotonicity.}
Plugging \eqref{eq:VW_star_closed_reused} into \eqref{eq:A_gamma_i_reused} with $\gamma=\gamma_{\gQ}$ gives
\[
A_{\gamma_{\gQ}}
=\frac{(a-\gamma_{\gQ})^k-b^k}{a^k-b^k},
\qquad
L_{\gQ}(f_k^\star)=d\,A_{\gamma_{\gQ}}^2
= d\left(\frac{(a-\gamma_{\gQ})^k-b^k}{a^k-b^k}\right)^2.
\]
To show strict monotonicity over even $k$, define
\[
r:=\left|\frac{a-\gamma_{\gQ}}{a}\right|,\qquad s:=\frac{b}{a}.
\]
Under the assumption $\gamma_{\gQ}\in(0,2a)\setminus\{\gamma_1,\gamma_2\}$ we have $r,s\in(0,1)$ and $r\neq s$.
For even $k$ (so $r^k,s^k>0$),
\[
|A_{\gamma_{\gQ}}|
=\left|\frac{r^k-s^k}{1-s^k}\right|.
\]
It suffices to show that the function
\[
\phi(t):=\left|\frac{r^t-s^t}{1-s^t}\right|,\qquad t>0,
\]
is strictly decreasing. Without loss of generality assume $r>s$ (the case $r<s$ is symmetric by absolute value).
Write $r=e^{-\alpha}$ and $s=e^{-\beta}$ with $0<\alpha<\beta$. Then
\[
\phi(t)=\frac{e^{-\alpha t}-e^{-\beta t}}{1-e^{-\beta t}}.
\]
A direct derivative calculation gives
\[
\phi'(t)
=
-\frac{\alpha(e^{\beta t}-1)-\beta(e^{\alpha t}-1)}{(e^{\beta t}-1)^2}\,e^{(\beta-\alpha)t}.
\]
Since $x\mapsto \frac{e^{xt}-1}{x}$ is strictly increasing for $x>0$ (for any fixed $t>0$), and $\beta>\alpha$,
we have $\alpha(e^{\beta t}-1)>\beta(e^{\alpha t}-1)$, hence the numerator is positive and thus $\phi'(t)<0$.
Therefore $\phi(t)$ (and hence $|A_{\gamma_{\gQ}}|$) decreases strictly with $t$, implying that $L_{\gQ}(f_k^\star)=dA_{\gamma_{\gQ}}^2$
decreases strictly as $k$ increases over even integers.

\paragraph{Step 6: Distance to the GD-implementing family.}
Let $(\mV^\star,\mW^\star)$ be as in \eqref{eq:VW_star_closed_reused}. By \eqref{eq:theta_gd}, the closest point in $\Theta_{\rm GD}$
matches $\mV^\star$ (take $\eta=1/[a(1-(b/a)^k)]$) and sets $\mW'=\mI_d$, hence
\[
\dist\big((\mV^\star,\mW^\star),\Theta_{\rm GD}\big)
=\|\mW^\star-\mI_d\|_F^2
=d\left(\frac{b}{a}\right)^{2k},
\]
proving item~(3).
\qed

\section{Proof of Theorem~\ref{thm:main}}

\begin{proof}
Fix \(k\). Let \((\hat g,\hat h)\) be any minimizer of \eqref{eq:erm}, and define
\(\hat f\coloneqq \hat g+\hat h\).

\paragraph{(1) ID upper bound.}
Recall
\[
\delta_k \coloneqq \min_{h\in\mathcal H_k}\Bigl(L_{\mathcal P}(h)+R_H(h)\Bigr).
\]
Since \(0\in\mathcal G_k\), for any \(h\in\mathcal H_k\) the pair \((0,h)\) is feasible for \eqref{eq:erm}, and
\(R(0,h)=R_G(0)+R_H(h)=R_H(h)\) because \(R_G(0)=0\).
By optimality of \((\hat g,\hat h)\),
\[
L_{\mathcal P}(\hat f)+R(\hat g,\hat h)
\;\le\;
L_{\mathcal P}(h)+R(0,h)
\;=\;
L_{\mathcal P}(h)+R_H(h)
\qquad \forall\,h\in\mathcal H_k.
\]
Taking the minimum over \(h\in\mathcal H_k\) yields
\[
L_{\mathcal P}(\hat f)+R(\hat g,\hat h)\le \delta_k.
\]
Since \(R(\hat g,\hat h)=R_G(\hat g)+R_H(\hat h)\ge 0\), we conclude
\[
L_{\mathcal P}(\hat f)\le \delta_k.
\]

\paragraph{(2) OOD lower bound.}
Define the residual associated with \(\hat g\):
\[
\hat r\coloneqq z-\hat g.
\]
Then
\[
L_{\mathcal Q}(\hat f)
=
\mathbb E_{x\sim\mathcal Q}\|\hat f(x)-z(x)\|^2
=
\mathbb E_{x\sim\mathcal Q}\|\hat g(x)+\hat h(x)-z(x)\|^2
=
\mathbb E_{x\sim\mathcal Q}\|\hat h(x)-\hat r(x)\|^2.
\]
By Assumption~\ref{as:decomp} (Shortcut part, \eqref{eq:shortcut_failure}), applied to the minimizer
\((\hat g,\hat h)\) of \eqref{eq:erm},
\[
\mathbb E_{x\sim\mathcal Q}\|\hat h(x)-\hat r(x)\|^2
\;\ge\;
C\,\mathbb E_{x\sim\mathcal Q}\|\hat r(x)\|^2
=
C\,L_{\mathcal Q}(\hat g).
\]
Using Assumption~\ref{as:decomp} (Well-structured part, \eqref{eq:shift_control_G}),
\(L_{\mathcal Q}(\hat g)\ge \frac{1}{\gamma}L_{\mathcal P}(\hat g)\).
Finally, since \(\hat g\in\mathcal G_k\),
\(
L_{\mathcal P}(\hat g)\ge \min_{g\in\mathcal G_k}L_{\mathcal P}(g)=\epsilon_k.
\)
Therefore,
\[
L_{\mathcal Q}(\hat f)\ge \frac{C}{\gamma}\,\epsilon_k = C_1\,\epsilon_k.
\]

\paragraph{(3) OOD upper bound.}
Let \(g_k^\star\in\argmin_{g\in\mathcal G_k}L_{\mathcal P}(g)\) be as in Assumption~\ref{as:decomp}(ii), so
\(L_{\mathcal P}(g_k^\star)=\epsilon_k\) and \(R_G(g_k^\star)\le \rho_k\).
Since \(0\in\mathcal H_k\), the pair \((g_k^\star,0)\) is feasible for \eqref{eq:erm}, and
\(R(g_k^\star,0)=R_G(g_k^\star)+R_H(0)=R_G(g_k^\star)\) because \(R_H(0)=0\).
By optimality of \((\hat g,\hat h)\),
\begin{equation}
\label{eq:obj_le_eps_rho}
L_{\mathcal P}(\hat f)+R_G(\hat g)+R_H(\hat h)
\;\le\;
L_{\mathcal P}(g_k^\star)+R_G(g_k^\star)
\;\le\;
\epsilon_k+\rho_k.
\end{equation}
Dropping the nonnegative term \(R_G(\hat g)\) gives
\begin{equation}
\label{eq:basic_bounds}
L_{\mathcal P}(\hat f)\le \epsilon_k+\rho_k,
\qquad
R_H(\hat h)\le \epsilon_k+\rho_k.
\end{equation}

We now bound \(L_{\mathcal Q}(\hat f)\). Using \(\|u+v\|^2\le 2\|u\|^2+2\|v\|^2\),
\[
L_{\mathcal Q}(\hat f)
=
\mathbb E_{x\sim\mathcal Q}\|(\hat g(x)-z(x))+\hat h(x)\|^2
\le
2L_{\mathcal Q}(\hat g)+2\,\mathbb E_{x\sim\mathcal Q}\|\hat h(x)\|^2.
\]
By \eqref{eq:shift_control_G}, \(L_{\mathcal Q}(\hat g)\le \gamma L_{\mathcal P}(\hat g)\). Also,
\[
L_{\mathcal P}(\hat g)
=
\mathbb E_{x\sim\mathcal P}\|(\hat f(x)-z(x))-\hat h(x)\|^2
\le
2L_{\mathcal P}(\hat f)+2\,\mathbb E_{x\sim\mathcal P}\|\hat h(x)\|^2.
\]
By Assumption~\ref{as:decomp}(i) (Norm control, \eqref{eq:norm_control}),
\[
\mathbb E_{x\sim\mathcal P}\|\hat h(x)\|^2 \le a_{\mathcal P}R_H(\hat h),
\qquad
\mathbb E_{x\sim\mathcal Q}\|\hat h(x)\|^2 \le a_{\mathcal Q}R_H(\hat h).
\]
Combining these bounds,
\begin{align*}
L_{\mathcal Q}(\hat f)
&\le
2\gamma\Bigl(2L_{\mathcal P}(\hat f)+2a_{\mathcal P}R_H(\hat h)\Bigr)
\;+\;2a_{\mathcal Q}R_H(\hat h) \\
&=
4\gamma\,L_{\mathcal P}(\hat f)
+\bigl(4\gamma a_{\mathcal P}+2a_{\mathcal Q}\bigr)R_H(\hat h).
\end{align*}
Applying \eqref{eq:basic_bounds} yields
\[
L_{\mathcal Q}(\hat f)
\le
\bigl(4\gamma+4\gamma a_{\mathcal P}+2a_{\mathcal Q}\bigr)\,(\epsilon_k+\rho_k)
=
C_2\,(\epsilon_k+\rho_k).
\]
Combining (1)--(3) completes the proof.
\end{proof}

\section{Proof of Proposition~\ref{prop:decomp_instantiation}}
\label{subsec:proof_prop_decomp_instantiation}

\begin{proof}
We work on the support \(\mathcal U=\{\pm1\}^p\) via \(u=\pi(x)=x_{1:p}\).
Recall \(\mathcal P=\mathrm{Unif}(\mathcal U_P)\), \(\mathcal Q=\mathrm{Unif}(\mathcal U_Q)\), and
\(|\mathcal U_P|=|\mathcal U_Q|=M=2^{p-1}\).
The target is \(z(u)=u^\top \theta\) with \(\theta\in\mathbb R_{\ge 0}^p\).

\paragraph{1) Distribution-shift control on \(\mathcal G_k\) with \(\gamma=1\), and the value of \(\epsilon_k\).}
For \(a(u)=\mathrm{diag}(w_1,\dots,w_p)u\) with \(0\le w_i\le \tau\), we have
\(a^{\circ k}(u)=\mathrm{diag}(w_1^k,\dots,w_p^k)u\). With \(\psi_1(v)=\mathbf 1^\top v\),
any \(g\in\mathcal G_k\) can be written as
\[
g(u)=\mathbf 1^\top a^{\circ k}(u)=\sum_{i=1}^p \alpha_i u_i,
\qquad \alpha_i\coloneqq w_i^k\in[0,\tau^k].
\]
Hence
\[
(g(u)-z(u))^2=\Big(\sum_{i=1}^p (\alpha_i-\theta_i)u_i\Big)^2.
\]
Under both \(\mathcal P\) and \(\mathcal Q\), the coordinates satisfy
\(\mathbb E[u_i u_j]=0\) for \(i\neq j\) and \(\mathbb E[u_i^2]=1\) for all \(i\) (note \(u_1\equiv \pm1\) but
still \(u_1^2=1\) and \(\mathbb E[u_1u_j]=\mathbb E[u_j]=0\) for \(j\ge 2\)).
Therefore,
\[
L_{\mathcal P}(g)=\sum_{i=1}^p (\alpha_i-\theta_i)^2
\;=\;L_{\mathcal Q}(g),
\]
so \eqref{eq:shift_control_G} holds with \(\gamma=1\).

To compute \(\epsilon_k\), we minimize \(\sum_{i=1}^p(\alpha_i-\theta_i)^2\) subject to \(\alpha_i\in[0,\tau^k]\),
which yields \(\alpha_i^\star=\min\{\theta_i,\tau^k\}\). Hence
\[
\epsilon_k=\min_{g\in\mathcal G_k}L_{\mathcal P}(g)
=\sum_{i=1}^p(\theta_i-\min\{\theta_i,\tau^k\})^2
=\sum_{i=1}^p\bigl((\theta_i-\tau^k)_+\bigr)^2.
\]

\paragraph{2) Norm control on \(\mathcal H_k\) with \(a_{\mathcal P}=a_{\mathcal Q}=1/(\lambda M)\).}
Here \(l=1\) and \(\psi_2(t)=t\), so any \(h\in\mathcal H_k\) is of the form \(h(u)=b(u)\) for some
\(b\in\mathcal B\). For the Kronecker delta kernel \(K(u,u')=\mathbf 1\{u=u'\}\) on the finite set \(\mathcal U\),
the RKHS norm satisfies
\[
\|b\|_{\mathcal B}^2=\sum_{u\in\mathcal U} b(u)^2.
\]
Moreover, the instantiation defines \(R_H(h)\coloneqq \lambda\|b\|_{\mathcal B}^2\).
Thus,
\[
\mathbb E_{u\sim\mathcal P}\,|h(u)|^2
=\frac{1}{M}\sum_{u\in\mathcal U_P} b(u)^2
\le \frac{1}{M}\sum_{u\in\mathcal U} b(u)^2
=\frac{1}{M}\|b\|_{\mathcal B}^2
=\frac{1}{\lambda M}R_H(h).
\]
The same argument holds for \(\mathcal Q\). Hence \eqref{eq:norm_control} holds with
\(a_{\mathcal P}=a_{\mathcal Q}=1/(\lambda M)\).

\paragraph{3) Bounded regularization for a structured fit and the value of \(\rho_k\).}
Under the instantiation, \(R_G(g)\coloneqq \lambda\|a\|^2\) with \(\|a\|^2=\sum_{i=1}^p w_i^2\).
Let \(g_k^\star\) be the minimizer achieving \(\alpha_i^\star=\min\{\theta_i,\tau^k\}\) above, and choose
\(w_i^\star\coloneqq (\alpha_i^\star)^{1/k}\in[0,\tau]\). Then \(g_k^\star\in\mathcal G_k\) and
\(L_{\mathcal P}(g_k^\star)=\epsilon_k\).
Moreover,
\[
R_G(g_k^\star)=\lambda\sum_{i=1}^p (w_i^\star)^2
=\lambda\sum_{i=1}^p \min\{\theta_i^{2/k},\tau^2\}
\le \lambda\sum_{i=1}^p \theta_i^{2/k}.
\]
Thus Assumption~\ref{as:decomp}(ii) holds with \(\rho_k=\lambda\sum_{i=1}^p \theta_i^{2/k}\).

\paragraph{4) Value of \(\delta_k\).}
Recall \(\delta_k=\min_{h\in\mathcal H_k}\bigl(L_{\mathcal P}(h)+R_H(h)\bigr)\).
Write \(h(u)=b(u)\). Since \(L_{\mathcal P}\) depends only on values on \(\mathcal U_P\), and \(R_H\) penalizes
\(\sum_{u\in\mathcal U}b(u)^2\), any minimizer satisfies \(b(u)=0\) for all \(u\notin\mathcal U_P\).
Thus the problem reduces to
\[
\min_{\{b(u)\}_{u\in\mathcal U_P}}
\frac{1}{M}\sum_{u\in\mathcal U_P}\bigl(b(u)-z(u)\bigr)^2
+\lambda\sum_{u\in\mathcal U_P} b(u)^2.
\]
This decouples over \(u\in\mathcal U_P\). For each \(u\), minimizing
\(\frac{1}{M}(b-z)^2+\lambda b^2\) gives \(b^\star=\frac{z}{1+\lambda M}\).
Plugging back yields the optimal value
\[
\delta_k
=\frac{\lambda}{1+\lambda M}\sum_{u\in\mathcal U_P} z(u)^2.
\]
Finally, under \(u\sim \mathrm{Unif}(\mathcal U_P)\),
\(\mathbb E[z(u)^2]=\theta^\top \mathbb E[uu^\top]\theta=\|\theta\|_2^2\), and hence
\(\sum_{u\in\mathcal U_P} z(u)^2 = M\|\theta\|_2^2\).
Therefore,
\[
\delta_k=\frac{\lambda M}{1+\lambda M}\,\|\theta\|_2^2.
\]

\paragraph{5) Shortcut failure on \(\mathcal H_k\) with \(C=1\).}
Let \((\hat g,\hat h)\) be any minimizer of \eqref{eq:erm}, and write \(\hat h(u)=\hat b(u)\).
As above, the objective depends on \(\hat b(u)\) for \(u\notin\mathcal U_P\) only through the penalty term
\(\lambda\sum_{u\in\mathcal U}\hat b(u)^2\). Therefore, for any minimizer we must have
\(\hat b(u)=0\) for all \(u\notin\mathcal U_P\), and in particular \(\hat h(u)=0\) for all \(u\in\mathcal U_Q\).
Define \(\hat r\coloneqq z-\hat g\). Then on \(\mathcal Q\),
\[
\hat h(u)-\hat r(u) = -\hat r(u),
\qquad u\in\mathcal U_Q,
\]
and hence
\[
\mathbb E_{u\sim\mathcal Q}\,|\hat h(u)-\hat r(u)|^2
=
\mathbb E_{u\sim\mathcal Q}\,|\hat r(u)|^2.
\]
Thus \eqref{eq:shortcut_failure} holds with \(C=1\).

Combining the above verifications establishes Assumption~\ref{as:decomp} with the stated constants and the
claimed expressions for \(\epsilon_k,\rho_k,\delta_k\).
\end{proof}

\section{Experimental Details}\label{apdx:exp_details}

\subsection{Additional Details for Section~\ref{subsec: phop}}

\subsubsection{Description of the $p$-hop Problem}\label{apdx:phop_def}

Formally, fix a finite alphabet $\Sigma$ and a sequence $v\in\Sigma^n$. Define the \emph{one-step backtracking} map
\[
\find_1(v,i)\;\coloneqq\;\max\Bigl(\{0\}\cup\{\,j\le i:\ v_{j-1}=v_i\,\}\Bigr),
\]
and for $p\ge 2$ define the $p$-step backtracking recursively by
\[
\find_p(v,i)\;\coloneqq\;\find_1\bigl(v,\find_{p-1}(v,i)\bigr).
\]
The $p$-hop target is then defined as $\hop_p(v)\coloneqq v_{\find_p(v,n)}$, and in our data generation we ensure that the $p$-hop always exists (i.e., $\find_p(v,n)\neq 0$).
Intuitively, the task asks the model to recursively backtrack through the sequence for $p$ steps according to the $\find_1$ rule, and then output the symbol found at the resulting location.

Following \cite{saunshi2025reasoning}, we generate sequences by first sampling a valid $p$-hop chain and then filling the remaining positions with random \emph{filler} symbols while respecting the $p$-hop constraints.

In our experiments, we set $p=4$, the vocabulary size to $4$, and the sequence length to $12$.

\subsubsection{Model and Training Details}
\label{apdx:phop_model_training}

\paragraph{Architecture.}
We use a looped Transformer classifier operating on fixed-length token sequences.
Each input sequence is embedded using learned token embeddings and learned positional embeddings.
The backbone consists of a single Transformer block, comprising multi-head self-attention followed by a two-layer MLP with GELU activations, residual connections, and layer normalization.
To scale latent reasoning depth, the backbone is applied repeatedly for $k\in\{1,2,4,8,10,12,14,16,24,28,32,36,40,44,48\}$ loops, with parameters shared across loops.
The final hidden state at the last sequence position is passed through a linear classification head to predict the output symbol.
The classification head is linear and does not use bias.
We use model dimension $256$, with $8$ attention heads and MLP hidden dimension $512$.

\paragraph{Optimization.}
Models are trained using AdamW with learning rate $10^{-3}$, cosine decay to a minimum learning rate of $10^{-4}$, and a warmup of $1000$ steps.
We set weight decay to zero.
Training is run for $20{,}000$ steps with batch size $256$.
Gradients are clipped to norm $1.0$.
We use automatic mixed precision for efficiency.

\begin{figure}[!t]
    \centering
    \includegraphics[width=0.32\linewidth]{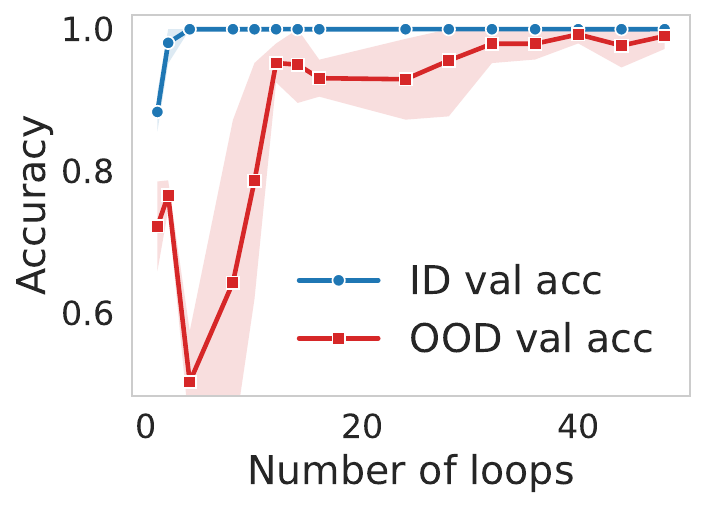}
    \caption{Results for looped Transformers on the $4$-hop task with loop embeddings show a pattern consistent with Figure~\ref{fig:phop}. ID accuracy saturates around 2 loops, while OOD accuracy continues improving up to around 36 loops.}
    \label{fig:phop_loop_embd}
\end{figure}

\subsubsection{Additional Experiments with Loop Embedding}\label{apdx:phop_loop_embedding}

In addition to the vanilla looped Transformer considered in Section \ref{subsec: phop}, we also study a variant that incorporates \emph{loop embeddings}.
Concretely, at each loop iteration $\ell\in\{0,\dots,k-1\}$, we add a learned embedding vector $e_k\in\mathbb{R}^{d_{\mathrm{model}}}$ to the hidden states before applying the shared Transformer block.
This provides the model with explicit information about the current iteration index, allowing the weight-tied backbone to implement iteration-dependent update rules.
This design is analogous to timestep or iteration encodings used in recurrent and looped Transformer variants, such as the Universal Transformer \cite{dehghani2018universal}, and more recent work on looped Transformers \cite{xu2024expressive} and continuous CoT models \cite{hao2024training}. Figure~\ref{fig:phop_loop_embd} reports results for looped Transformers with loop embeddings. We observe the same pattern as in Figure~\ref{fig:phop}: ID validation accuracy saturates rapidly around 2 loops, while OOD accuracy continues to improve with increasing loop count, exhibiting clear gains up to around $36$ loops.
This indicates that the OOD improvements observed in Section~\ref{subsec: phop} are robust to the inclusion of explicit loop-index conditioning and do not rely on strict homogeneity across iterations.

\subsection{Additional Details for Section~\ref{subsec: rl}}
\label{apdx:rl}

\subsubsection{Data and Training Details}
\label{apdx:rl_data_training}

\paragraph{Additional dataset details.}
For both the ID and OOD datasets, we use 900 examples for training and 100 examples for validation. We create a custom variant of the \texttt{Polynomial Roots} task because the original version is too difficult for the 1.5B base model (0\% accuracy), making RL training uninformative. In our variant, each example presents a univariate polynomial $P(x)$ (with integer coefficients) and asks for the \emph{largest rational root} written in lowest terms as $n/m$, followed by the derived target $|n|\cdot|m|$. Concretely, we construct polynomials by sampling a small number of rational roots and coefficient factors, forming linear terms of the form $(A_i x - b_i)$, and multiplying them using \texttt{sympy} to obtain $P(x)$, which is then used in the prompt. The ground-truth label is computed by identifying the largest rational root of $P(x)$, reducing it to coprime integers $(n,m)$, and outputting $|n|\cdot|m|$. For example, one instance from our dataset is:
\begin{quote}\small
Let $P(x)=144x^3-276x^2-104x-9$. Find the largest root $n/m$ where $n$ and $m$ are coprime. What is $|n|\cdot|m|$?
\end{quote}
The largest root is $9/4$, so the target is $|9|\cdot|4|=36$.
This construction yields instances that are solvable by the 1.5B model while still requiring multi-step reasoning.

\paragraph{Training and evaluation setup.}
We fine-tune \texttt{Qwen2.5-1.5B-Instruct} using GRPO~\cite{guo2025deepseek} with a correctness-only reward, and apply
LoRA~\cite{hu2022lora} with rank $32$ and $\alpha=64$.
All experiments are run using the \texttt{verl} framework.
We train for 4 epochs (224 total training steps) with AdamW, a learning rate of $3\times10^{-6}$, and a
KL regularization coefficient of $0.001$.
The batch size is 16, with 16 generations sampled per prompt.
Training is performed on 4 NVIDIA A6000 GPUs.
At evaluation time, we decode with temperature $1$.

\begin{figure}[!t]
    \centering
    \includegraphics[width=0.6\linewidth]{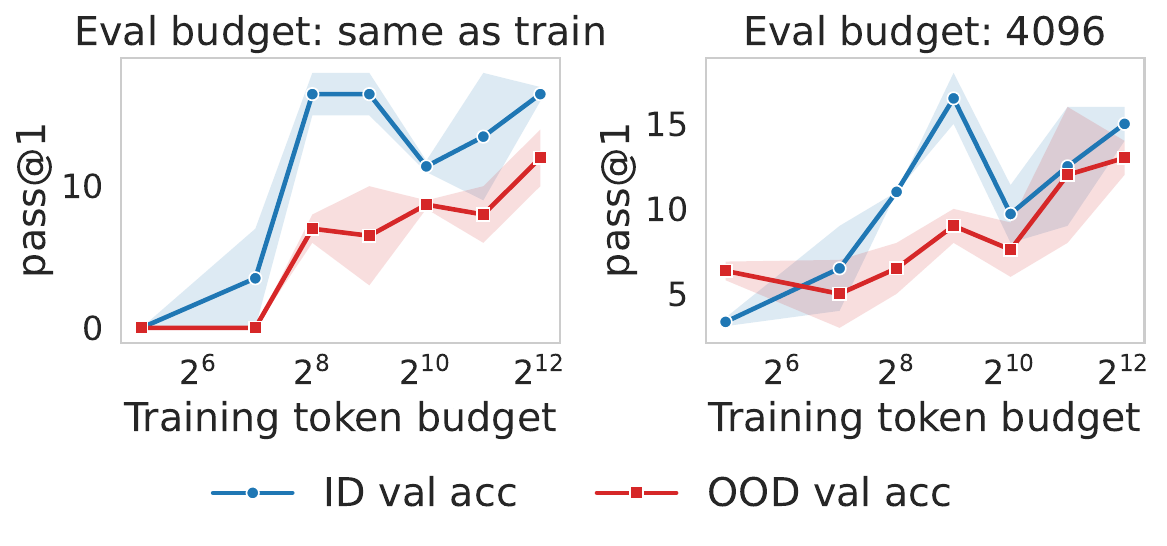}
    \vspace{-.2cm}
\caption{Pass@1 results show the same qualitative pattern: ID performance oscillates slightly more but exhibits no clear improvement beyond a token budget of 256, while OOD accuracy continues to improve steadily under both evaluation setups.\looseness=-1}
    \vspace{-.3cm}
    \label{fig:rl_poly_pass1}
\end{figure}

\subsubsection{Additional Results}
\label{apdx:rl_additional}

The Pass@1 counterpart to Figure~\ref{fig:rl_poly} is shown in Figure~\ref{fig:rl_poly_pass1}. ID accuracy still peaks around a token budget of 256, while OOD accuracy continues to improve steadily under both evaluation setups.

\begin{examplebox}{Example outputs at different training-time token budgets}
\noindent\textbf{Question.}
\begin{modelout}
For 5x**3 + 13x*2 + 4x - 4 = 0, express the smallest root as n/m in lowest terms. Calculate the product of absolute values |n| * |m|?    
\end{modelout}
\vspace{0.5em}
\noindent\textbf{Model trained with a token budget of 128.}
\begin{modelout}
\\[\n\\boxed{2}\n\\]
\end{modelout}

\vspace{0.5em}
\noindent\textbf{Model trained with a token budget of 256.}
\begin{modelout}
We first find the roots of the given cubic equation by factoring or using numerical methods. This gives us the roots -4, -2, and 1/2. The smallest root is -2, which is in the form of a fraction. Therefore, n=2 and m=1, and their product of absolute values is |n| * |m| = 2 * 1 = \\boxed{2}.
\end{modelout}

\vspace{0.5em}
\noindent\textbf{Model trained with a token budget of 4096.}
\begin{modelout}
(Continued on next page.)
\end{modelout}
\begin{modelout}
To find the smallest root of the polynomial equation \\(5x^3 + 13x^2 + 4x - 4 = 0\\), we can use the Rational Root Theorem, which suggests that any
    rational root of the polynomial equation is a factor of the constant term (-4) divided by a factor of the leading coefficient (5). The factors of 4 are \\(\\pm 1, \\pm 2,
    \\pm 4\\), and the factors of 5 are \\(\\pm 1, \\pm 5\\). Therefore, the possible rational roots are \\(\\pm 1, \\pm \\frac{1}{5}, \\pm 2, \\pm \\frac{2}{5}, \\pm 4, \\pm
    \\frac{4}{5}\\).\n\nWe will test the possible roots by substituting them into the polynomial:\n\nFirst, let's test \\(x = 1\\):\n\\[5(1)^3 + 13(1)^2 + 4(1) - 4 = 5 + 13 + 4
     - 4 = 18 \\neq 0.\\]\nSo, \\(x = 1\\) is not a root.\n\nNext, let's test \\(x = -1\\):\n\\[5(-1)^3 + 13(-1)^2 + 4(-1) - 4 = -5 + 13 - 4 - 4 = -2 \\neq 0.\\]\nSo, \\(x =
    -1\\) is not a root.\n\nNext, let's test \\(x = \\frac{1}{5}\\):\n\\[5\\left(\\frac{1}{5}\\right)^3 + 13\\left(\\frac{1}{5}\\right)^2 + 4\\left(\\frac{1}{5}\\right) - 4 =
    5\\left(\\frac{1}{125}\\right) + 13\\left(\\frac{1}{25}\\right) + 4\\left(\\frac{1}{5}\\right) - 4 = \\frac{5}{125} + \\frac{13}{25} + \\frac{4}{5} - 4.\\]\nConverting all
    terms to a common denominator (125):\n\\[\\frac{5}{125} + \\frac{13(5)}{125} + \\frac{80}{125} - \\frac{500}{125} = \\frac{5 + 65 + 80 - 500}{125} = \\frac{-350}{125} =
    -2.8 \\neq 0.\\]\nSo, \\(x = \\frac{1}{5}\\) is not a root.\n\nNext, let's test \\(x = -\\frac{1}{5}\\):\n\\[5\\left(-\\frac{1}{5}\\right)^3 +
    13\\left(-\\frac{1}{5}\\right)^2 + 4\\left(-\\frac{1}{5}\\right) - 4 = 5\\left(-\\frac{1}{125}\\right) + 13\\left(\\frac{1}{25}\\right) + 4\\left(-\\frac{1}{5}\\right) - 4
    = -\\frac{5}{125} + \\frac{13}{25} - \\frac{4}{5} - 4.\\]\nConverting all terms to a common denominator (125):\n\\[-\\frac{5}{125} + \\frac{13(5)}{125} - \\frac{80}{125} -
    \\frac{500}{125} = -\\frac{5}{125} + \\frac{65}{125} - \\frac{80}{125} - \\frac{500}{125} = -\\frac{580}{125} = -4.64 \\neq 0.\\]\nSo, \\(x = -\\frac{1}{5}\\) is not a
    root.\n\nNext, let's test \\(x = 2\\):\n\\[5(2)^3 + 13(2)^2 + 4(2) - 4 = 5 \\cdot 8 + 13 \\cdot 4 + 4 \\cdot 2 - 4 = 40 + 52 + 8 - 4 = 96 \\neq 0.\\]\nSo, \\(x = 2\\) is
    not a root.\n\nNext, let's test \\(x = -2\\):\n\\[5(-2)^3 + 13(-2)^2 + 4(-2) - 4 = 5 \\cdot (-8) + 13 \\cdot 4 + 4 \\cdot (-2) - 4 = -40 + 52 - 8 - 4 = 0.\\]\nSo, \\(x =
    -2\\) is a root.\n\nSince \\(x = -2\\) is a root, we can perform polynomial division of \\(5x^3 + 13x^2 + 4x - 4\\) by \\(x + 2\\). Using synthetic
    division:\n\n\\[\n\\begin{array}{r|rrrr}\n-2 & 5 & 13 & 4 & -4 \\\\\n   &   & -10 & -6 & 4 \\\\\n\\hline\n   & 5 & 3 & -2 & 0 \\\\\n\\end{array}\n\\]\nThe quotient is
    \\(5x^2 + 3x - 2\\), so we have:\n\\[5x^3 + 13x^2 + 4x - 4 = (x + 2)(5x^2 + 3x - 2).\\]\nNext, we solve the quadratic equation \\(5x^2 + 3x - 2 = 0\\) using the quadratic
    formula \\(x = \\frac{-b \\pm \\sqrt{b^2 - 4ac}}{2a}\\):\n\\[x = \\frac{-3 \\pm \\sqrt{3^2 - 4 \\cdot 5 \\cdot (-2)}}{2 \\cdot 5} = \\frac{-3 \\pm \\sqrt{9 + 40}}{10} =
    \\frac{-3 \\pm \\sqrt{49}}{10} = \\frac{-3 \\pm 7}{10}.\\]\nThis gives us two solutions:\n\\[x = \\frac{-3 + 7}{10} = \\frac{4}{10} = \\frac{2}{5},\\]\n\\[x = \\frac{-3 -
    7}{10} = \\frac{-10}{10} = -1.\\]\nThe roots of the equation \\(5x^3 + 13x^2 + 4x - 4 = 0\\) are \\(x = -2\\) and \\(x = \\frac{2}{5}\\) and \\(x = -1\\). The smallest root
     is \\(-2\\), so \\(n = -2\\) and \\(m = 1\\).\n\nThe product of the absolute values of \\(n\\) and \\(m\\) is:\n\\[|n| \\cdot |m| = |-2| \\cdot |1| = 2 \\cdot 1 =
    2.\\]\nThus, the answer is:\n\\[\\boxed{2}.\\]
\end{modelout}
\end{examplebox}

\end{document}